\documentclass[12pt]{colt2019}

\makeatletter

\def\noindentation{\let\@afterindentfalse}

\newcommand{\mytitle}[1]{
    \vskip 2em
    {\bf\sffamily\LARGE #1}
    \noindentation}

\renewcommand{\@maketitle}{
    \begin{flushleft}{

        \@author \par 

        \@date \par 

        \mytitle{\@title}

    } \end{flushleft}
} 
\makeatother

\usepackage[utf8]{inputenc}
\usepackage[T1]{fontenc}

\usepackage{microtype}
\usepackage{booktabs} 

\usepackage{amsfonts} 
\usepackage{bm}
\usepackage{mathtools} 

\usepackage{setspace}
\usepackage{xspace}

\usepackage{xcolor} 
\usepackage{wrapfig}
\usepackage{array} 
\usepackage{multirow}
\usepackage{enumitem}
\usepackage{mdframed}
\usepackage{relsize}
\usepackage{caption}

\usepackage{cleveref}
\creflabelformat{equation}{#2#1#3}

\usepackage[colorinlistoftodos, textwidth=20mm, disable]{todonotes}
\definecolor{blued}{RGB}{70,197,221}
\newcommand{\todod}[1]{\todo[color=blued,inline]{#1}}

\newcommand{\todoaout}[1]{\todo[color=yellow]{\tiny#1}}
\definecolor{citrine}{rgb}{0.89, 0.82, 0.04}
\newcommand{\todom}[1]{\todo[color=citrine]{\tiny#1}}
\newcommand{\todomi}[1]{\todo[color=citrine, inline]{#1}}
\renewcommand{\cite}[1]{\citep{#1}}

\renewcommand{\epsilon}{\varepsilon}

\newcommand{\ie}{i.e.\xspace}
\newcommand{\eg}{e.g.\xspace}
\newcommand{\whp}{w.h.p.\xspace}

\newcommand{\abw}{\wt{\bw}}

\newcommand{\abA}{\wt{\bA}}

\newcommand{\phimat}[1]{\boldsymbol{\Phi}(\bX_{#1})}
\newcommand{\aphimat}[1]{\wt{\boldsymbol{\Phi}}_{#1}(\bX_{#1})}
\newcommand{\phivec}[1]{\featmap(\bx_{#1})}

\newcommand{\nystrom}{Nystr\"{o}m\xspace}

\newcommand{\coldict}{\mathcal{S}}

\newcommand{\pvec}{\mathbf{\psi}}

\newcommand{\pmat}{\mathbf{\Psi}}
\newcommand{\armset}{\mathcal{A}}
\newcommand{\featmap}{\boldsymbol{\phi}}
\newcommand{\embfunc}{\bz}
\newcommand{\embvec}{\bz}
\newcommand{\embmat}{\bZ}

\newcommand{\fnorm}{F}
\newcommand{\Narm}{A}
\newcommand{\history}{\mathcal{D}}
\newcommand{\qbar}{\wb{q}}

\newcommand{\bkb}{\textsc{BKB}\xspace}
\newcommand{\oful}{\textsc{OFUL}\xspace}
\newcommand{\soful}{\textsc{SOFUL}\xspace}
\newcommand{\gpucb}{\textsc{GP-UCB}\xspace}

\newcommand{\supkernelucb}{\textsc{SupKernelUCB}\xspace}

\def\:#1{\protect \ifmmode {\mathbf{#1}} \else {\textbf{#1}} \fi}

\newcommand{\wt}[1]{\widetilde{#1}}
\newcommand{\wh}[1]{\widehat{#1}}

\newcommand{\wb}[1]{\overline{#1}}

\newcommand{\bsym}[1]{\mathbf{#1}}

\newcommand{\ba}{\mathbf{a}}

\newcommand{\bk}{\mathbf{k}}

\newcommand{\bw}{\mathbf{w}}
\newcommand{\bx}{\mathbf{x}}
\newcommand{\by}{\mathbf{y}}
\newcommand{\bz}{\mathbf{z}}

\newcommand{\bA}{\mathbf{A}}

\newcommand{\bC}{\mathbf{C}}

\newcommand{\bG}{\mathbf{G}}

\newcommand{\bI}{\mathbf{I}}

\newcommand{\bK}{\mathbf{K}}

\newcommand{\bP}{\mathbf{P}}

\newcommand{\bS}{\mathbf{S}}

\newcommand{\bV}{\mathbf{V}}

\newcommand{\bX}{\mathbf{X}}

\newcommand{\bZ}{\mathbf{Z}}

\DeclareMathOperator*{\argmin}{arg\,min}
\DeclareMathOperator*{\argmax}{arg\,max}
\DeclareMathOperator*{\logdet}{log\,det}

\renewcommand{\epsilon}{\varepsilon}
\newcommand{\bigotime}{\mathcal{O}}
\newcommand{\abigotime}{\wt{\mathcal{O}}}

\newcommand{\normsmall}[1]{\Vert #1 \Vert}
\newcommand{\normempty}[1]{\left\Vert #1 \right\Vert}

\newcommand{\transp}{\mathsf{\scriptscriptstyle T}}
\newcommand{\comp}{\complement}
\DeclareMathOperator*{\Tr}{Tr}
\DeclareMathOperator*{\Det}{Det}
\DeclareMathOperator*{\Ran}{Im}

\newcommand{\probability}{\mathbb{P}}

\DeclareMathOperator*{\expectedvalue}{\mathbb{E}}

\newcommand{\condbar}{\;\middle|\;}
\newcommand{\gaussdistr}{\mathcal{N}}

\newcommand{\bernoullidist}{{\rm Bernoulli}}
\newcommand{\filtration}{\mathcal{F}}

\newcommand{\indfunc}{\mathbb{I}}
\newcommand{\Real}{\mathbb{R}}

\newcommand{\rkhs}{\mathcal{H}}

\newcommand{\kerfunc}{k}
\newcommand{\kermatrix}{{\bK}}

\newcommand{\deff}{d_{\text{\rm eff}}}

\newcommand{\selmatrix}{{\bS}}

\let\originalleft\left
\let\originalright\right
\renewcommand{\left}{\mathopen{}\mathclose\bgroup\originalleft}
\renewcommand{\right}{\aftergroup\egroup\originalright}

\title[BKB]{
Gaussian Process Optimization with Adaptive Sketching:\\*
Scalable and No Regret 
}
\usepackage{times}

\coltauthor{\Name{Daniele Calandriello} \Email{daniele.calandriello@iit.it}\\
 \addr LCSL - Istituto Italiano di Tecnologia, Genova, Italy \& MIT, Cambridge, USA
 \AND
 \Name{Luigi Carratino} \Email{luigi.carratino@dibris.unige.it}\\
 \addr UNIGE - Universit\`a degli Studi di Genova, Genova, Italy
 \AND
 \Name{Alessandro Lazaric} \Email{lazaric@fb.com}\\
 \addr FAIR - Facebook AI Research, Paris, France
 \AND
 \Name{Michal Valko} \Email{michal.valko@inria.fr}\\
 \addr INRIA Lille - Nord Europe, SequeL team, Lille, France
 \AND
 \Name{Lorenzo Rosasco} \Email{lrosasco@mit.edu}\\
 \addr LCSL - Istituto Italiano di Tecnologia, Genova, Italy \& MIT, Cambridge, USA\\*
UNIGE - Universit\`a degli Studi di Genova, Genova, Italy
}
\begin{document}

\maketitle

\begin{abstract}
Gaussian processes (GP) are a well studied Bayesian approach for the optimization of black-box functions. Despite their effectiveness in simple problems, GP-based algorithms hardly scale to high-dimensional functions, as their per-iteration time and space cost is at least \emph{quadratic} in the number of dimensions $d$ and iterations~$t$. Given a set of $\Narm$ alternatives to choose from, the overall runtime $\bigotime(t^3\Narm)$ is prohibitive. 
In this paper we introduce \bkb (\textit{budgeted kernelized bandit}), a new approximate GP algorithm for optimization under bandit feedback that achieves near-optimal regret (and hence near-optimal convergence rate) with near-constant per-iteration complexity and remarkably no assumption on the input space or covariance of the GP.

We combine a kernelized linear bandit algorithm (GP-UCB) with randomized matrix sketching based on leverage score sampling, and we prove that randomly sampling inducing points based on their posterior variance gives an accurate low-rank approximation of the GP, preserving variance estimates and confidence intervals. As a consequence, \bkb does not suffer from \emph{variance starvation}, an important problem faced by many previous sparse GP approximations.
Moreover, we show that our procedure
	selects at most $\wt{\bigotime}(\deff)$ points, where $\deff$
	is the \emph{effective} dimension of the explored space, which is typically much
	smaller than both $d$ and $t$. 
	This greatly reduces the dimensionality of the
	problem, thus leading to a $\bigotime(T\Narm\deff^2)$ runtime and $\bigotime(
    \Narm \deff)$ space complexity.
\end{abstract}

\begin{keywords}
    \emph{sparse Gaussian process optimization};
    \emph{kernelized linear bandits};
    \emph{regret};
    \emph{sketching};
    \emph{Bayesian optimization};
    \emph{black-box optimization};
    \emph{variance starvation}
\end{keywords}

\vspace{-.5\baselineskip}
\section{Introduction}
Efficiently selecting the best alternative out of a set of alternatives is  important
 in sequential decision making, 
with practical applications ranging from recommender systems \citep{li2010contextual}
to experimental design~\citep{robbins1952some}.
It is also the main focus of the research in bandits~\citep{lattimore2019bandit} and Bayesian
optimization~\citep{mockus1989global,pelikan2005hierarchical,snoek2012practical},
that study optimization under bandit feedback. 
In this setting, a learning algorithm sequentially interacts with
a reward or utility function $f$. Over $T$ interactions, the algorithm chooses a point $\bx_t$ and it has only access to a noisy black-box evaluation of $f$ at $\bx_t$. 
The goal of the algorithm is to minimize the cumulative regret, which compares the reward accumulated at the points selected over time, $\sum_t f(\bx_t),$ to the reward obtained by repeatedly selecting the optimum of the function, \ie, $T\max_x f(x).$
In this paper take the Gaussian process optimization approach.
In particular, we study the \gpucb algorithm first introduced by~\citet{srinivas2010gaussian}.

Starting from a Gaussian process prior over $f,$ \gpucb alternates between evaluating the function,
and using the evaluations to build a posterior of $f$.
This posterior is composed by a mean function $\mu$ that estimates
the value of $f$, and a variance function $\sigma$ that captures the uncertainty~$\mu$. 
These two quantities are combined in a single upper confidence bound (UCB)
that drives the selection of the evaluation points, and trades off between
evaluating high-reward points (\emph{exploitation}) and testing possibly sub-optimal points to reduce the uncertainty on the function (\emph{exploration}).
While other approaches to select promising points exist, such as expected improvement (EI)
and maximum probability of improvement, it is not known if they can achieve
low regret.
The performance of \gpucb has been studied by~ \citet{srinivas2010gaussian,valko2013finite,chowdhury2017kernelized}  to show that \gpucb provably achieves low regret both in a Bayesian and
non-Bayesian setting. However, the main limiting factor to its applicability is its computational cost.
When choosing between $\Narm$ alternatives, \gpucb requires $\Omega(\Narm t^2)$ time/space to select each new point,
and does not scale to long and complex optimization problems.
Several approximations of \gpucb have been suggested \citep{quinonero-candela2007approximation,liu2018gaussian}
and we review them next:\\[0.025in]
\textbf{Inducing points:} The GP can be restricted to lie in the range
of a small subset of inducing points. The subset should cover the space well
for accuracy, but also be as small as possible for efficiency.
Methods referred to as \textit{sparse GPs}, have been proposed to select the inducing points and an approximation based on the subset.
Popular instances of this approach are the subset of regressors (SoR, \citealp{wahba1990spline}) and the deterministic training conditional (DTC, \citealp{seeger2003fast}).
While these methods are simple to interpret and efficient, they do not come
with regret guarantees. Moreover, when the subset does not cover
the space well, they suffer from \emph{variance starvation} \citep{wang2018batched}, as they underestimate the variance
of points far away from the inducing points.\\[0.025in]
\textbf{Random Fourier features:}
Another approach is to use explicit feature expansions to approximate the GP
covariance function, and embed the points in a low-dimensional space,
usually exploiting some variation of Fourier expansions~\citep{rahimi2007random}.
Among these methods, \citet{mutny2018efficient} recently showed that
discretizing the posterior on a fine grid of quadrature Fourier features (QFF) incurs
a negligible approximation error. This is sufficient
to prove that the maximum of the approximate posterior can be efficiently found
and that it is accurate enough to guarantee that Thompson sampling with QFF provably achieves low regret.
However this approach does not extend to non-stationary (or non-translation invariant) kernels and although its dependence on $t$ is small,
the approximation and posterior maximization procedure scales exponentially with the input dimension.\\[0.025in]
\textbf{Variational inference:}  This approach replaces the true GP
likelihood with a variational approximation that can be optimized efficiently.
Although recent methods provide guarantees on the approximate posterior mean
and variance \citep{huggins2019scalable}, these guarantees only apply to GP regression and not to the harder optimization setting.\\[0.025in]
\textbf{Linear case:} There are multiple methods that reduce the complexity of linear bandits algorithms, most focused on approximating LinUCB \cite{li2010contextual}. \citet{kuzborskij2019efficient} uses the \emph{frequent directions} (FD, \citealp{ghashami2016frequent}) 
to project the design matrix data to a smaller subspace. Unfortunately, the size of the subspace has to be specified in advance, and when the size is not sufficiently large the method suffers linear regret.
Prior to the FD approach, \textsc{CBRAP} \citep{yu2017contextual} used random projections instead, but faced similar issues.
This turns out to be a fundamental weakness of all approaches that do not adapt the \textit{actual} size of the space defined by the sequence of points selected by the learning algorithm.
Indeed, \citet{ghosh2017misspecified} showed a lower bound
that shows that
as soon as one single arm does not abide by the projected linear model
 we can suffer linear regret.

\vspace{-.5\baselineskip}
\subsection{Contributions}
In this paper, we show a way to adapt the size of the projected
space online and devise the \bkb (budgeted kernel bandit) algorithm achieving near-optimal regret with a computational complexity drastically smaller than \gpucb. 
This is achieved \emph{without assumptions on the complexity}
of the input or on the kernel function.
\bkb leverages several well-known tools: a DTC approximation
of the posterior variance, based on inducing points, and a confidence interval construction
based on state-of-the-art self-normalized concentration inequalities~\citep{abbasi2011improved}.
It also introduces two novel tools: a selection strategy to select
inducing points based on ridge leverage score (RLS) sampling \cite{alaoui2014fast} that is provably accurate,
and an approximate confidence interval  that is not only
nearly as accurate as the one of \gpucb, but also efficient.

Ridge leverage score sampling was introduced for randomized kernel matrix approximation
\citep{alaoui2014fast}. In the context of GPs, RLS correspond to the posterior variance
of a point, which allows adapting algorithms and their guarantees 
from the RLS sampling to the GP setting. This solves two 
problems in sparse GPs and linear bandit approximations.
First, \bkb constructs estimates of the variance that are provably accurate,
\ie, it does not suffer from variance starvation, which results in provably
accurate confidence bounds as well.
The only method with comparable guarantees, Thompson sampling with quadrature FF
\citep{mutny2018efficient}, only applies to stationary  kernels, and is applicable only when the input is low dimensional
or the covariance $\kerfunc$ has an additive structure.
Moreover our approximation guarantees are qualitatively different since they do not require a corresponding \emph{uniform} approximation bound on the GP.
Second, \bkb adaptively chooses the size of the inducing point set based
on the \emph{effective dimension} $\deff$ of the problem, also known as degrees
of freedom of the model. This is crucial to achieve low regret, since
fixed approximation schemes may suffer linear regret.
Moreover, in a problem with $A$ arms, using a set of $\bigotime(\deff)$ inducing points results in an algorithm with $\bigotime(\Narm\deff^2)$ per-step runtime and $\bigotime(\Narm\deff)$ space, a significant improvement
over the $\bigotime(\Narm t^2)$ time and $\bigotime(\Narm t)$ space cost of \gpucb.

Finally, while in our work we only address kernelized (GP) bandits, 
our work could be extended to more complex online learning problems,
such as to recent advances in kernelized reinforcement learning
\citep{chowdhury2019online}.
Moreover, inducing point methods have clear interpretability and our analysis
provides insight both from a bandit and Bayesian optimization perspective,
making it applicable to a large amount of downstream tasks.

 \vspace{-.5\baselineskip}
\section{Background}\label{sec:background}

\textbf{Notation.} We use
lower-case letters $a$ for scalars,
lower-case bold letters $\ba$ for vectors,
and upper-case bold letters $\bA$ for matrices and operators,
where $[\bA]_{ij}$ denotes its element $(i,j)$.
We denote by $\normsmall{\bx}_{\bA}^2 \triangleq \bx^\transp\bA\bx,$ the norm with metric $\bA$,
and $\normsmall{\bx} \triangleq \normsmall{\bx}_{\bI}$ with $\bI$ being the identity. Finally, we denote the first $T$ integers as $[T] \triangleq \{1,\ldots,T\}$.

\paragraph{Online optimization under bandit feedback.}
Let $f: \armset \rightarrow \Real$ be a reward function that we wish to optimize
over a set of decisions $\armset$, also called actions or arms. 
For simplicity, we assume that $\armset \triangleq \{\bx_i\}_{i=1}^{\Narm}$ is a fixed finite set of $\Narm$
vectors in $\Real^d$. We discuss how to relax these assumptions in \Cref{sec:discussion}.
\todoaout{We overload $x_t$ and $x_i$...}
In optimization under bandit feedback, a learner aims to optimize $f$
through a sequence of interactions.
At each step $t \in [T],$ the learner \textbf{(1)} chooses an arm $\bx_t \in \armset$, \textbf{(2)} receives reward $y_t \triangleq f(\bx_t) + \eta_t$, where $\eta_t$ is a zero-mean noise, and \textbf{(3)} updates its model of the problem.

The goal of the learner is to minimize its cumulative
regret 
$R_T \triangleq \sum_{t=1}^T f(\bx_\star) - f(\bx_t)$
w.r.t.\,the best\footnote{We assume a consistent and arbitrary tie-breaking strategy.} $\bx_\star$,
where $\bx_\star \triangleq \argmax_{\bx_i \in \armset} f(\bx_i)$. In particular, the objective of a \textit{no-regret} algorithm is to have $R_T/T$ go to zero as fast as possible when $T$ grows. Recall that the regret is  related to the convergence rate and the optimization performance. In fact, let $\wb \bx_T$ be an arm chosen at random from the sequence of arms $(\bx_1, \ldots, \bx_T)$ selected by the learner, then $f(\bx_\star) - \mathbb{E}[f(\bx_T)] \leq R_T/T$.\todoaout{add ref.}

\subsection*{Gaussian process optimization and \gpucb}\label{ssec:gpucb}
\gpucb is popular no-regret algorithm for optimization under bandit feedback and was
introduced by \citet{srinivas2010gaussian} for Gaussian process
optimization. We first give the formal definition of a Gaussian process  \citep{rasmussen2006gaussian},
and then briefly present \gpucb.

A \emph{Gaussian process} ${\rm {\rm GP}}(\mu, \kerfunc)$  is a generalization of the Gaussian
distribution to a space of functions and it is defined by a mean function $\mu(\cdot)$
and covariance function $\kerfunc(\cdot, \cdot).$ We consider zero-mean ${\rm GP}(0, \kerfunc)$ priors and bounded covariance $\kerfunc(\bx_i, \bx_i) \leq \kappa^2$ for all $\bx_i \in \armset$.
An important property of Gaussian processes is that if we combine a prior
$f \sim {\rm GP}(0, \kerfunc)$ and assume that the observation noise is zero-mean Gaussian (i.e., $\eta_t \sim \gaussdistr(0,\xi^2)$),
then the posterior distribution of $f$
conditioned on a set of observations $\{(\bx_s, y_s)\}_{s=1}^t$ is also a {\rm GP}. More precisely, if $\bX_t  \triangleq [\bx_1, \dots, \bx_t]^\transp \in \Real^{t \times d}$ is the matrix with all arms selected so far and $\by_t \triangleq [y_1, \dots, y_t]^\transp$  the corresponding observations, then the posterior is still a ${\rm GP}$ and the mean and variance of the function at a test point $\bx$ are defined as
\begin{align}
&\mu_t\left(\bx \condbar \bX_t, \by_t\right) \!=\! \bk_{t}(\bx)^\transp(\bK_{t} + \lambda\bI)^{-1}\by_{t}\label{eq:gpucb-posterior-mean},\\
&\sigma_{t}^2\left(\bx \condbar \bX_t\right) \!=\! \kerfunc(\bx, \bx) - \bk_{t}(\bx)^\transp(\bK_{t} + \lambda\bI)^{-1}\bk_{t}(\bx)\label{eq:gpucb-posterior-variance},
\end{align}
where $\lambda  \triangleq \xi^2$, $\bK_t \in \Real^{t \times t}$ is the matrix \todoaout{kernel?}
$[\bK_t]_{i,j} \triangleq \kerfunc(\bx_i, \bx_j)$ constructed from all pairs
$\bx_i,\bx_j$ in $\bX_t$, and
$\bk_t(\bx) \triangleq [\kerfunc(\bx_1, \bx), \dots, \kerfunc(\bx_t, \bx)]^\transp$. Notice that $\bk_t(\bx)$ can be seen as an \textit{embedding} of an arm $\bx$ represented using by the arms $\bx_1,\ldots,\bx_T$ observed so far.

The \emph{\gpucb algorithm}
uses a Gaussian process
${\rm GP}(0, \kerfunc)$ as a prior for~$f$. 
Inspired by the optimism-in-face-of-uncertainty principle, at each time step~$t$, \gpucb uses the posterior ${\rm GP}$ to compute the mean and variance of an arm $\bx_i$ and obtain the score
\begin{align}\label{eq:gpucb.score}
u_{t}(\bx_i) \triangleq \mu_t(\bx_i) + \beta_t\sigma_t(\bx_i),
\end{align}
where we use the short-hand notation $\mu_t(\cdot) \triangleq \mu\left(\,\cdot \condbar \bX_t, \by_t\right)$ and $\sigma_t(\cdot) \triangleq \sigma\left(\,\cdot \condbar \bX_t\right)$. Finally, \gpucb chooses the maximizer $\bx_{t+1} \triangleq \argmax_{\bx_i \in \armset} u_{t}(\bx_i)$
as the next arm to evaluate. According to the score $u_t$, an arm $\bx$ is likely to be selected if it has high mean reward $\mu_t$ \textit{or} high variance $\sigma_t$, i.e., its estimated reward $\mu_t(\bx)$ is very uncertain. As a result, selecting the arm $\bx_{t+1}$ with the largest score trades off between collecting (estimated) large reward (\textit{exploitation}) and improving the accuracy of the posterior (\textit{exploration}). The parameter $\beta_t$ balances between these two objectives and must be properly tuned to guarantee
low regret.
\citet{srinivas2010gaussian} proposes different
approaches for tuning~$\beta_t,$ depending on the assumptions on $f$ and $\armset$. 

While \gpucb is interpretable, simple to implement and provably achieves low regret, it is computationally expensive. In particular,
computing $\sigma_t(\bx)$ has a complexity at least $\Omega(t^2)$
for the matrix-vector product $(\bK_{t-1} + \xi^2\bI)^{-1}\bk_{t-1}(\bx)$.
Multiplying this complexity by $T$ iterations and $\Narm$ arms results in an
overall $\bigotime(AT^3)$  cost, which does not scale to large number of iterations $T$.

 \section{Budgeted Kernel Bandits}\label{sec:bkb}

In this section, we introduce the \bkb (\textit{budgeted kernel bandit}) algorithm, a novel efficient approximation of \gpucb, and we provide guarantees for its computational complexity. The analysis in \Cref{sec:bkb-regret-analysis} shows that \bkb can be tuned to significantly reduce the complexity of \gpucb with a negligible impact on the regret. 
We begin by introducing the two major contributions of this section:
an approximation of the GP-UCB scores supported only by a small subset $\coldict_t$ of \emph{inducing points},
and a method to \textit{incrementally and adaptively} construct an accurate subset $\coldict_t$.

\subsection{The algorithm}\label{sec:bkb-algorithm}
The main complexity bottleneck to compute the scores in \Cref{eq:gpucb.score}
is due to the fact that after~$t$ steps, the posterior GP is  supported on \emph{all} $t$ previously seen arms.
As a consequence, evaluating \Cref{eq:gpucb-posterior-mean,eq:gpucb-posterior-variance}
requires computing a~$t$ dimensional vector $\bk_t(\bx)$ and $t \times t$
matrix $\bK_t$ respectively.
To avoid this dependency we restrict both $\bk_t$ and $\bK_t$
to be supported on a \emph{subset} $\coldict_t$ of $m$ arms.
This approach is a case of the sparse Gaussian process approximation \citep{quinonero-candela2007approximation},
or equivalently, linear bandits constrained to a subspace \citep{kuzborskij2019efficient}.

\textbf{Approximated GP-UCB scores.} Consider a subset of arm $\coldict_t \triangleq \{\bx_i\}_{i=1}^m$ 
and let $\bX_{\coldict_t} \in \Real^{m \times d}$ be the matrix with all arms in $\coldict_t$ as rows. Let $\bK_{\coldict_t} \in \Real^{m \times m}$ be the matrix constructed by evaluating the covariance $k$ between any two pairs of arms in $\coldict_t$ and 
$\bk_{\coldict_t}(\bx) \triangleq [\kerfunc(\bx_1, \bx), \dots, \kerfunc(\bx_m, \bx)]^\transp$.
The \nystrom embedding $\bz_t(\cdot)$
associated with subset $\coldict_t$ is defined as the mapping\footnote{Recall that in the exact version, $\bk_t(\bx)$ can be seen as an embedding of any arm $\bx$ into the space induced by all the $t$ arms selected so far, \ie using all selected points as inducing points.}
\begin{align*}
\vspace{-0.3\baselineskip}
\embfunc_{t}(\cdot) \triangleq \left(\bK_{\coldict_t}^{1/2}\right)^{+}\bk_{\coldict_t}(\cdot) : \Real^d \rightarrow \Real^m,
\end{align*}
where $(\cdot)^{+}$ indicates the pseudo-inverse. We denote with $\embmat_t(\bX_t) \triangleq [\embfunc_{t}(\bx_1), \dots, \embfunc_{t}(\bx_t)]^\transp \in \Real^{t \times m}$ the  associated matrix of points and we define
$\bV_t \triangleq \embmat_t(\bX_t)^\transp\embmat_t(\bX_t) + \lambda\bI$. Then, we approximate the posterior mean, variance, and UCB for the value of the function at $\bx_i$ as
\begin{align}\label{eq:approx.gpucb.score}
\wt{\mu}_t(\bx_i) &\triangleq \embvec_t(\bx_i)^\transp\bV_t^{-1}\embmat_t(\bx_i)^\transp\by_t,\nonumber\\
\wt{\sigma}_t^2(\bx_i) &\triangleq \frac{1}{\lambda}\Big(\kerfunc(\bx_i,\bx_i)
- \embvec_t(\bx_i)^\transp\embmat_t(\bX_t)^\transp\embmat_t(\bX_t)\bV_t^{-1}\embvec_t(\bx_i)\Big),\nonumber\\
\wt{u}_{t}(\bx_i) &\triangleq \wt{\mu}_t(\bx_i) + \wt{\beta}_t\wt{\sigma}_t(\bx_i),
\vspace{-0.3\baselineskip}
\end{align}
where $\wt{\beta}_t$ is appropriately tuned to achieve small regret in the theoretical analysis of \Cref{sec:bkb-regret-analysis}. Finally, at each time step $t$, \bkb selects arm $\wt{\bx}_{t+1} = \argmax_{\bx_i \in \armset}\wt{u}_{t}(\bx_i)$.

Notice that in general, $\wt{\mu}_t$ and $\wt{\sigma}_t$ do \emph{not} correspond to any
GP posterior. In fact, if we were simply replacing the $k(\bx_i,\bx_i)$ in the expression of $\wt{\sigma}_t^2(\bx_i)$ by its value in the \nystrom embedding, i.e., $\bz_t(\bx_i)^\transp \bz_t(\bx_i)$, then we would recover a sparse GP approximation known as the \emph{subset of regressors}.
Using $\bz_t(\bx_i)^\transp \bz_t(\bx_i)$ is known to cause \emph{variance starvation},
as it can severely underestimate the variance
of a test point $\bx_i$ when it is far from the points in $\coldict_t$.
Our formulation of $\wt{\sigma}_t$ is known in
Bayesian world as the \emph{deterministic
training conditional} (DTC), where it is used as a heuristic to prevent variance starvation.
However, DTC does \emph{not} correspond to a GP since it violates consistency \citep{quinonero-candela2007approximation}.
In this work, we justify this approach rigorously, showing that it is crucial
to prove approximation guarantees necessary both for the optimization process and
for the construction of the set of inducing points.

\begin{wrapfigure}[14]{R}{0.52\textwidth}
\vspace{-2.35\baselineskip}
\begin{algorithm2e}[H]\SetAlgoLined
\KwData{Arm set $\armset$, $q$, $\{\beta_t\}_{t=1}^T$}
\KwResult{Arm choices $\history_T = \{(\wt{\bx}_t, y_t)\}$}
Select uniformly at random $\bx_1$ and observe $y_1$\;
Initialize $\coldict_1 = \{\bx_1\}$\;
\For{$t=\{1, \dots, T-1\}$}{
    Compute $\wt{\mu}_{t}(\bx_i)$ and $\wt{\sigma}_{t}^2(\bx_i)$ for all $\bx_i \in \armset$\;
    Select $\wt{\bx}_{t+1} = \argmax_{\bx_i \in \armset}\wt{u}_{t}(\bx_i)$ (Eq.~\ref{eq:approx.gpucb.score})\;
    \For{$i=\{1,\dots,t+1\}$}{
        Set $\wt{p}_{t+1, i} = \qbar\cdot\wt{\sigma}_t^2(\wt{\bx}_{i})$\;
        Draw $q_{t+1,i} \sim \bernoullidist\left(\wt{p}_{t+1,i}\right)$\label{algline:zdraw}\;
        If $q_{t+1} = 1$ include $\wt{\bx}_i$ in $\coldict_{t+1}$\label{algline:expandop}\;
    }
}
\caption{
\bkb\label{alg:bkb}}
\end{algorithm2e}
\end{wrapfigure}
\paragraph{Choosing the inducing points.}
A critical aspect to effectively keep the complexity of \bkb low while still controlling the regret is to carefully choose the inducing points to
include in the subset $\coldict_t$. 
As the complexity of computing~$\wt{u}_t$ scales with the size $m$
of $\coldict_t$, a smaller set gives a faster algorithm.
Conversely, the difference between $\wt{\mu}_t$ and $\wt{\sigma}_t$ and
their exact counterparts depends on the accuracy of the embedding~$\bz_t$, which  increases with the size of the set~$\coldict_t$.
Moreover, even for a fixed $m$, the quality of the embedding greatly depends on \emph{which}
inducing points are included. For instance, selecting 
the same arm as inducing point twice, or two co-linear arms, does not
improve accuracy as the embedding space does not change.
Finally, we need to take into account two important aspects of sequential
optimization when choosing~$\coldict_t$. First, we need to focus our approximation
more on regions of $\armset$ that are relevant to the objective (i.e., high-reward arms). 
Second, as these regions change over time, we need to
keep adapting the composition and size of $\coldict_t$ accordingly.

To address the first objective, we choose to construct
$\coldict_t$ by randomly subsampling only out of the set of arms $\wt{\bX}_t$
evaluated so far. This set will naturally focus on high-reward arms,
as low-reward arms will be selected increasingly less often and will become
a small minority of $\wt{\bX}_t$.
To address the change in focus over time, arms are selected for inclusion in~$\coldict_t$ with a probability proportional to their posterior variance $\sigma_t$
at step $t$, which changes accordingly.
We report the selection procedure in \Cref{alg:bkb},
with the complete \bkb algorithm.

We initialize $\coldict_1 \triangleq \{\wt{\bx}_1\}$ by
selecting an arm uniformly at random. At each
step $t$, after selecting $\wt{\bx}_{t+1}$, we must regenerate
$\coldict_t$ to reflect the changes in $\wt{\bX}_{t+1}$ (\ie resparsify the GP approximation).
Ideally, we would sample each arm in $\wt{\bX}_{t+1}$ proportionally
to $\sigma_{t+1}^2$, but this would be too computationally expensive.
Therefore, we apply two approximations. First we approximate $\sigma_{t+1}^2$ with
$\sigma_{t}^2$. This is equivalent to ignoring the last
arm and does not significantly impact the accuracy.
We can then replace $\sigma_{t}^2$ with $\wt{\sigma}_{t}^2$
which can be computed efficiently, and in practice we simply cache and reuse the $\wt{\sigma}_{t}^2$ already computed when constructing \Cref{eq:approx.gpucb.score}.
Finally, given a parameter $\qbar \geq 1$, we
set our approximate inclusion probability as $\wt{p}_{t+1,i} \triangleq \qbar\wt{\sigma}_t^2(\wt{\bx}_{s})$.
The $\qbar$ parameter is used to increase the inclusion probability
in order to boost the overall success probability of the approximation procedure
at the expense of a small increase in the size of $\coldict_{t+1}$.
Given $\wt{p}_{t+1,i}$, we start from an empty $\coldict_{t+1}$
and iterate over all $\wt{\bx}_i$ for $i \in [t+1]$ drawing $q_{t+1,i}$
from a Bernoulli distribution with probability $\wt{p}_{t+1,i}$.
If $q_{t+1,i} = 1$, $\wt{\bx}_i$ is included in $\coldict_{t+1}$.

Notice that while constructing $\coldict_t$ based on $\sigma_t^2$ is
a common heuristic for sparse GPs,
it has not been yet rigorously justified.
In the next section, we show that this posterior variance sampling approach is equivalent to
$\lambda$-ridge leverage score (RLS) sampling \citep{alaoui2014fast},
a well studied tool in randomized linear algebra. We leverage the known results from this field to prove both accuracy and efficiency guarantees for our selection procedure.

\subsection{Complexity analysis}\label{sec:bkb-complexity-analysis}
\begin{figure*}[!t]
\centering
\minipage{0.32\textwidth}
  \includegraphics[width=\linewidth]{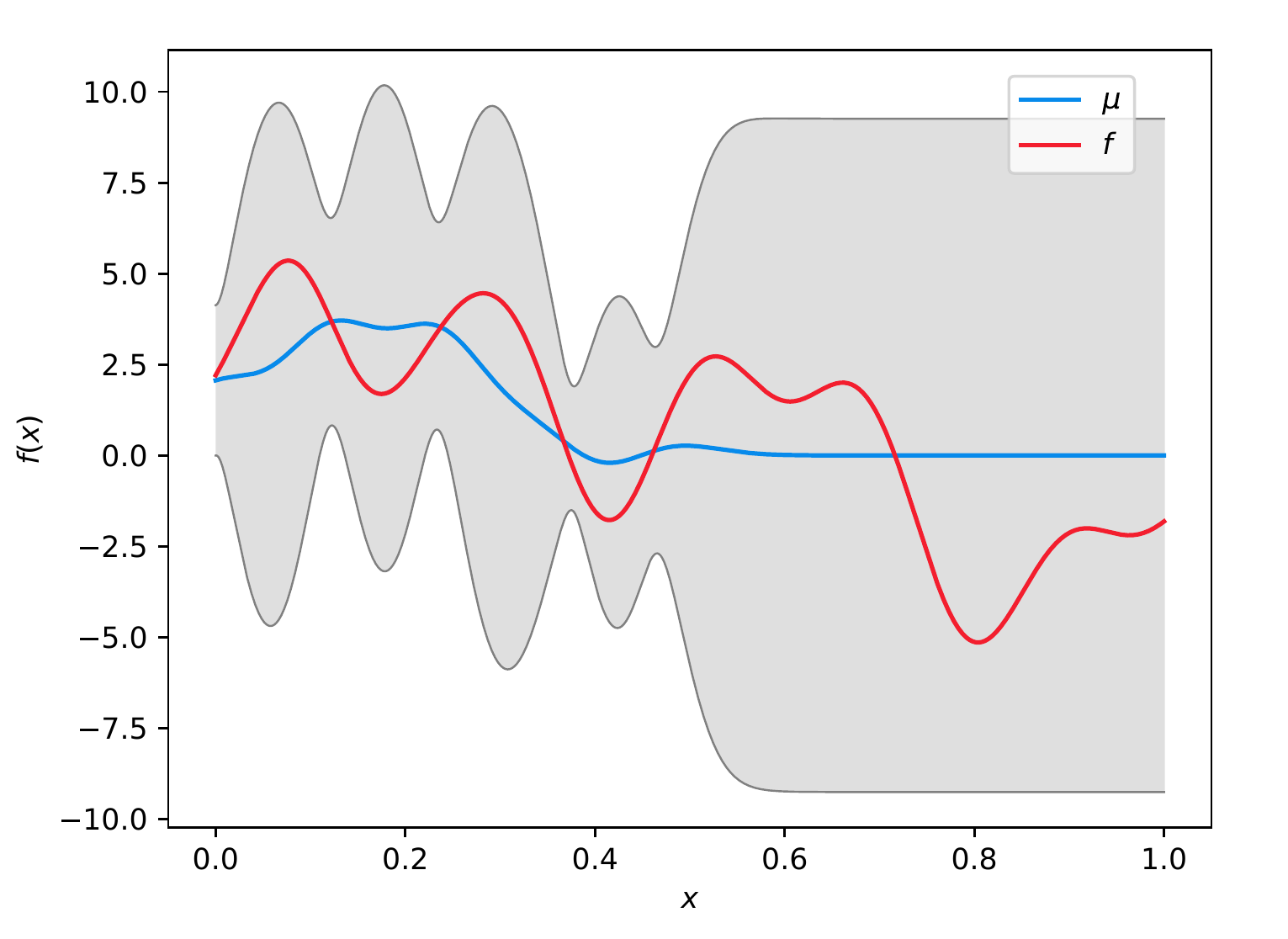}
\endminipage\hfill
\minipage{0.32\textwidth}
  \includegraphics[width=\linewidth]{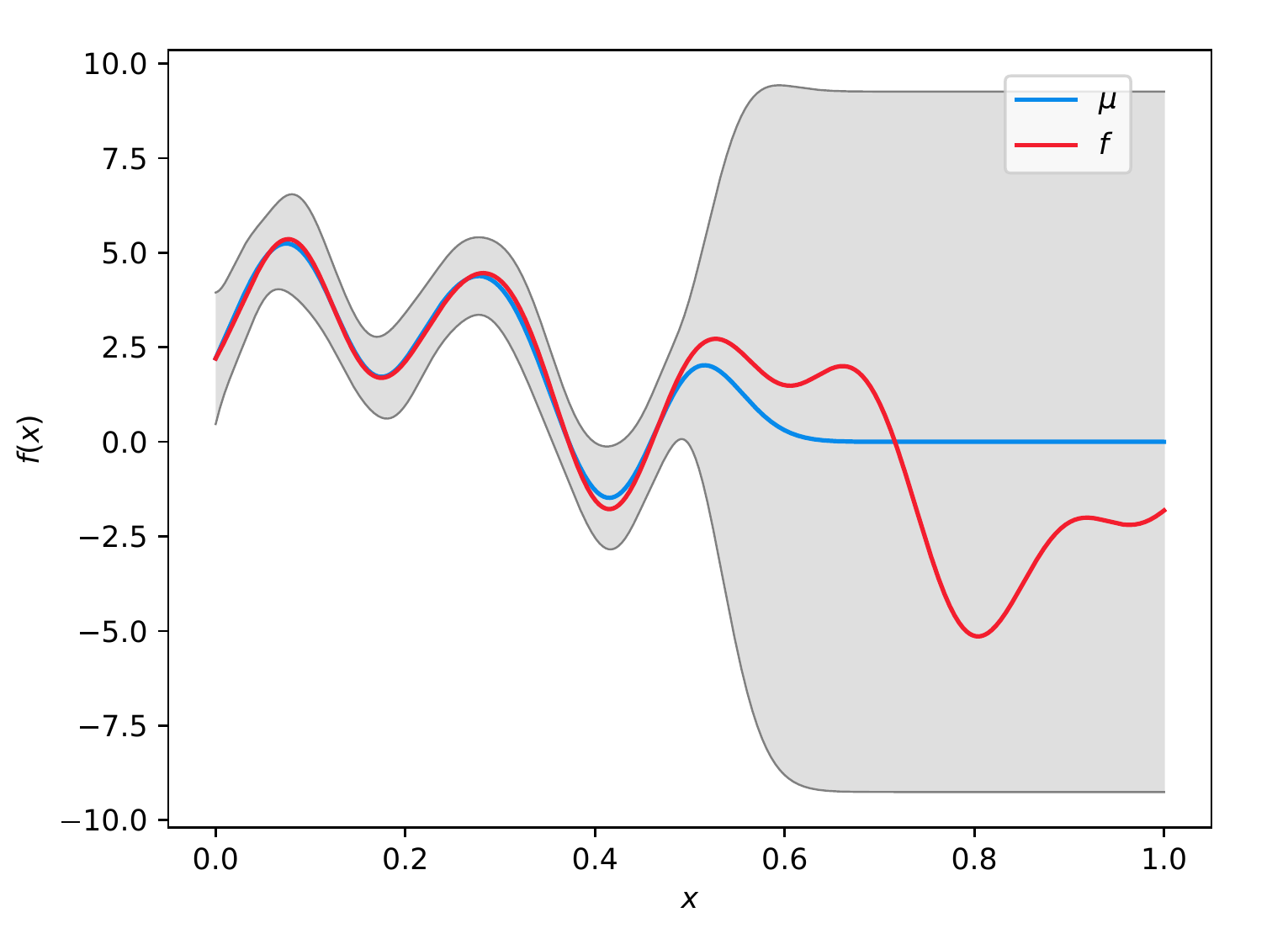}
\endminipage\hfill
\minipage{0.32\textwidth}
  \includegraphics[width=\linewidth]{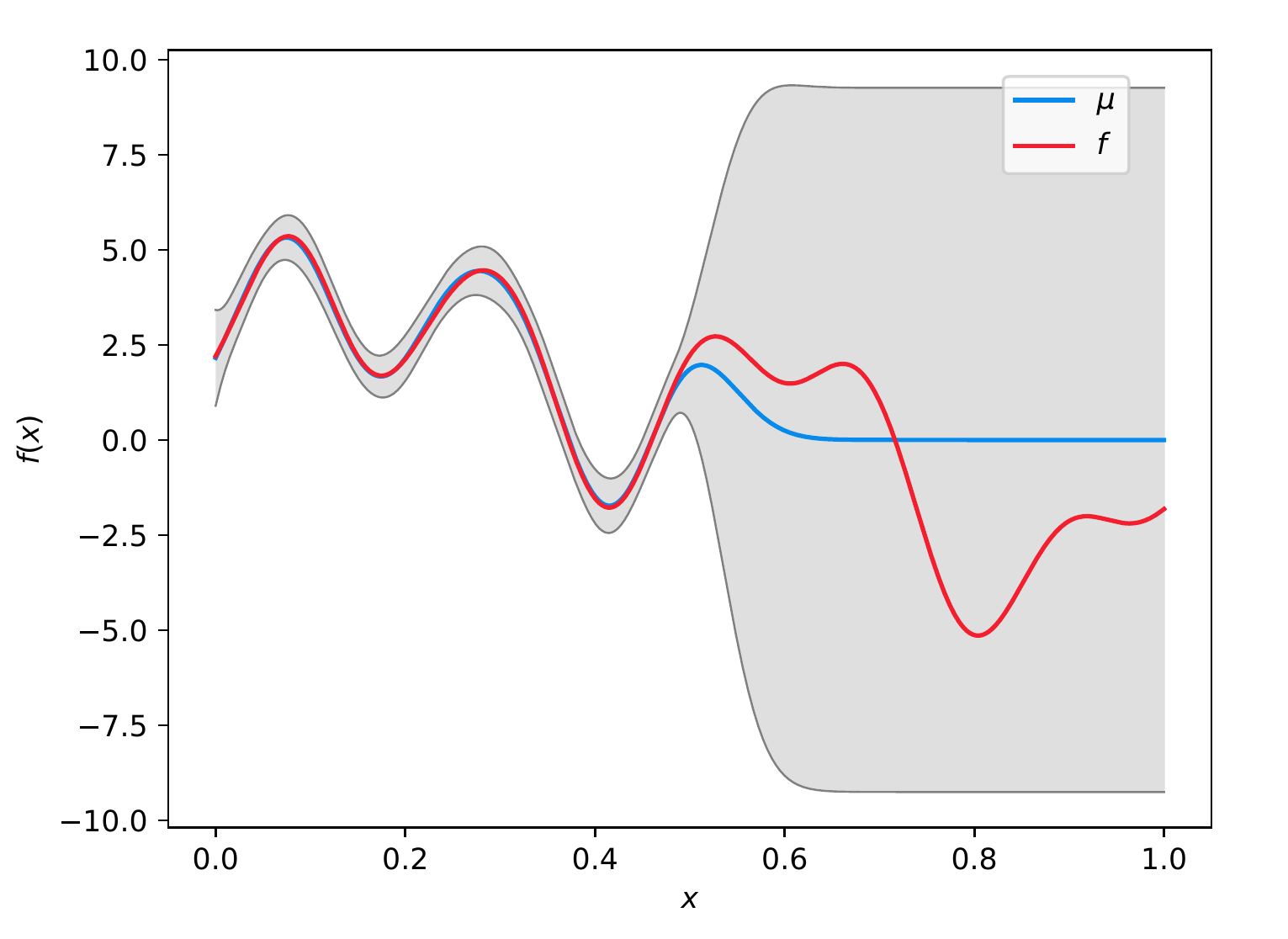}
\endminipage
\vspace{-.5\baselineskip}
\caption{
{\small We simulate a GP on $[0,1] \in \Real$ using Gaussian kernel with bandwidth $\sigma^2 \triangleq100$. We draw $f$ from the GP and give to \bkb $t \in \{6, 63, 215\}$ evaluations sampled uniformly in $[0, 0.5]$. We plot $f$ and $\wt{\mu}_t \pm 3\wt{\sigma}_t$.}
\label{fig:gp-var-starvation}}
\end{figure*}
Let $m_t \triangleq |\coldict_t|$ be the size of the set $\coldict_t$ at step $t$.
At each step,
we first compute the embedding $\embfunc_t(\bx_i)$ of all arms in $\bigotime(\Narm m_t^2 + m_t^3)$ time, which corresponds to one inversion of $\bK_{\coldict_t}^{1/2}$ and the matrix-vector product specific to each arm.
We then rebuild the matrix $\bV_t$ from scratch using all the arms observed so far. In general, it is sufficient to use counters to record the arms pulled so far, rather than the full list of arms, so that $\bV_t$ can be constructed in $\bigotime(\min\{t, \Narm\} m_t^2)$ time. Then, the inverse $\bV_t^{-1}$ is computed in $\bigotime(m_t^3)$ time. We can now efficiently compute 
$\wt\mu_t$, $\wt\sigma_t$, and $\wt{u}_{t}$ for all arms in $\bigotime(\Narm  m_t^2)$ time reusing the embeddings and $\bV_{t}^{-1}$. Finally, computing all $q_{t+1,i}$s and $\coldict_{t+1}$
takes $\bigotime(\min\{t + 1, \Narm\})$ time using the estimated variances $\wt{\sigma}_{t}^2$.
As a result, the per-step complexity is of order $\bigotime\big((A + \min\{t, \Narm\}) m_T^2\big)$.\footnote{Notice that $m_t \leq \min\{t,a\}$ and thus the complexity term $\bigotime(m_t^3)$ is absorbed by the other terms.}
Space-wise, we only need to store the embedded arms and~$\bV_t$ matrix, which takes at most $\bigotime(\Narm m_T)$ space.\\
\noindent
\textbf{The size of $\coldict_T$.}
The size $m_t$ of $\coldict_t$ can be expressed using the $q_{t,i}$ r.v.\@ as the sum
$m_t \triangleq \sum_{i=1}^t q_{t,i}$. In order to provide a bound on the total number of
inducing points, which directly determines the computational complexity of
\bkb, we go through three major steps.

The first is to show that \whp, $m_t$
is close to the sum $\sum_{i=1}^t \wt{p}_{t,i} = \sum_{i=1}^t\qbar\wt{\sigma}_t^2(\wt{\bx}_i)$,
\ie, close to the sum of the probabilities we used to sample
each $q_{t,i}$. However, the different $q_{t,i}$ are \emph{not independent}
and each $\wt{p}_{t,i}$ is itself a r.v. Nonetheless all $q_{t,i}$
are conditionally independent given the previous $t-1$ steps,
and this is sufficient to obtain the result.

The second and a more complex step is to guarantee that the
random sum $\sum_{i=1}^t \wt{\sigma}_t^2(\wt{\bx}_i)$ is close to
$\sum_{i=1}^t \sigma_t^2(\wt{\bx}_i)$ and, at a lower level, that
each individual estimate $\wt{\sigma}_t^2(\cdot)$ is close to
$\sigma_t^2(\cdot)$.
To achieve this we exploit the connection between ridge leverage scores and
posterior variance $\sigma_t^2$. In particular, we show that
the variance estimator $\wt{\sigma}_t^2(\cdot)$ used by \bkb
is a variation of the RLS estimator of \citet{calandriello2017distributed}
for RLS sampling. As a consequence,
we can transfer the strong accuracy and size guarantees of RLS sampling
to our optimization setting (see \Cref{sec:app-proof-kors}).\todoaout{Add ref to appendix.}
Note that anchoring the probabilities to the RLS (\ie the sum of the posterior variances) means that the size of $\coldict_t$ naturally follows the effective dimension of the arms pulled so far.
This strikes an adaptive balance between decreasing each individual probability
to avoid $\coldict_t$ growing too large, while at the same time automatically increasing the effective degrees of freedom of the sparse GP when necessary.

The first two steps lead to $m_t \approx \sum_{i=1}^t \sigma_i^2(\wt{\bx}_i)$, for which we need to derive a more explicit bound.
In the GP analyses, this quantity is bounded using the maximal information gain
$\gamma_T$ after~$T$ rounds. For this, let $\bX_{\armset} \in \Real^{\Narm \times d}$ be
the matrix with all arms as rows,  $\history$ a subset of
these rows, potentially with duplicates, and $\kermatrix_{\history}$
the associated kernel matrix. Then, \citet{srinivas2010gaussian} define
\vspace{-0.3\baselineskip}
\[\gamma_T \triangleq \max_{\history \subset \armset: |\history| = T} \tfrac{1}{2}\logdet(\kermatrix_{\history}/\lambda + \bI),
\vspace{-0.2\baselineskip}
\]
and show that $\sum_{i=1}^t \sigma_i^2(\wt{\bx}_i) \leq \gamma_t$,
and that $\gamma_T$ itself can be bounded for specific $\armset$ and kernel functions, \eg, $\gamma_T \leq \bigotime(\log(T)^{d+1})$ for
Gaussian kernels.
Using the equivalence between RLS and posterior variance $\sigma_t^2$, we can also relate the posterior variance $\sigma_t^2(\wt{\bx}_i)$ of the evaluated arms
to the so-called GP's \emph{effective dimension} $\deff$ or degrees of freedom
\begin{align}\label{eq:def-deff}
\vspace{-0.3\baselineskip}
\deff(\lambda, \wt{\bX}_T) \triangleq \sum\nolimits_{i=1}^t \sigma_t^2(\wt{\bx}_i) = \Tr(\kermatrix_T(\kermatrix_T + \lambda\bI)^{-1}),
\vspace{-0.1\baselineskip}
\end{align}
using the following inequality by \citet{calandriello2017second},
{
\begin{align}\label{eq:deff-logdet-main-text}
\logdet\left(\kermatrix_T/\lambda + \bI\right)
    \leq \Tr(\kermatrix_T(\kermatrix_T + \lambda\bI)^{-1})\left(1 +\log\left(\tfrac{\normsmall{\kermatrix_T}}{\lambda} + 1\right)\right).
\end{align}
}
We use both RLS and $\deff$ to describe \bkb's selection.
We now give the main result of this section.
\begin{theorem}\label{thm:kors-complexity}
For a desired $0 < \varepsilon < 1$, $0 < \delta < 1$, let $\alpha \triangleq (1 + \varepsilon)/(1 - \varepsilon)$.
If we run \bkb with
$\qbar \geq 6\alpha\log(4T/\delta)/\varepsilon^2,$ 
then with probability $1-\delta$, for all $t \in [T]$ and for all $\bx \in \armset,$ we have
\begin{align*}
\sigma_t^2(\bx)/\alpha \leq \wt{\sigma}_t^2(\bx) \leq \alpha\sigma_t^2(\bx)&&
\text{and} && |\coldict_t| \leq 3(1+\kappa^2/\lambda)\alpha\qbar\deff(\lambda, \wt{\bX}_t).
\end{align*}
\end{theorem}
\textbf{Computational complexity.} We already showed that \bkb's implementation
with \nystrom embedding requires $\bigotime(T(\Narm + \min\{t, \Narm\}) m_T^3)$ time and $\bigotime(\Narm m_T)$ space.
Combining this with \Cref{thm:kors-complexity} and the bound $m_T \leq \abigotime(\deff)$,
we obtain a $\abigotime(T\Narm \deff^2 + \min\{t, \Narm\}) \deff^3)$ time complexity. Whenever $\deff \ll T$ and 
$T \ll \Narm$, this is essentially a quadratic $\bigotime(T^2)$ runtime, a
large improvement over the quartic $\bigotime(T^4) \leq \bigotime(T^3\Narm)$ runtime of \gpucb.\\*
\textbf{Tuning $\qbar$.} Note that although $\qbar$ must satisfy the condition
of \Cref{thm:kors-complexity} for the result to hold, it is quite robust
to uncertainty on the desired horizon $T$. In particular, the bound
holds for \emph{any} $\varepsilon > 0$, and even if we continue
updating $\coldict_T$ after the $T$-th step, the bound still holds by
implicitly increasing the parameter $\varepsilon$.
Alternatively, after the $T$-th iteration the user can suspend the algorithm,
increase $\qbar$ to suit the new desired horizon,
and rerun only the subset selection on the arms selected so far. \\ \noindent
\textbf{Avoiding variance starvation.} Another important consequence of \Cref{thm:kors-complexity}
is that \bkb's variance estimate is always close to the exact one up to a small constant factor.
To the best of our knowledge, it makes \bkb the first efficient and general
GP algorithm that provably avoids variance starvation, which can be caused by two
sources of error. The first source is the degeneracy, \ie, low-rankness of the GP approximation
which causes the estimate to grow over-confident when the number of observed
points grows and exceeds the degrees of freedom of the GP.
\bkb \emph{adaptively chooses its degrees of freedom} as the size of $\coldict_t$ scales with the effective dimension.
The second source of error arises when a point is far away from $\coldict_t$.
Our use of a DTC variance estimator avoids under-estimation before we update
the subset $\coldict_t$. Afterward, we can use guarantees on the quality of~$\coldict_t$
to guarantee that we do not over-estimate the variance too much, exploiting a similar approach used
to guarantee accuracy in RLS estimation.
Both problems, and \bkb's accuracy, are highlighted in \Cref{fig:gp-var-starvation}
using a benchmark experiment proposed by \citet{wang2018batched}.\\*
\textbf{Incremental dictionary update.} At each step $t$, \bkb recomputes the dictionary $\coldict_{t+1}$ from scratch by sampling each of the arms pulled so far with a suitable probability $\wt p_{t+1,i}$. A more efficient variant would be to build $\coldict_{t+1}$ by adding the new point $\: x_{t+1}$ with probability $\wt p_{t+1,t+1}$ and including the points in $\coldict_{t}$ with probability $\wt p_{t+1,i}/\wt p_{t,i}$. This strategy is used in the streaming setting 
to avoid storing all points observed so far and incrementally update the dictionary~\citep[see][]{calandriello2017distributed}. Nonetheless, the stream of points, although arbitrary, is assumed to be generated \textit{independently} from the dictionary itself. On the other hand, in our bandit setting, the points $\wt{\:x}_1, \wt{\:x}_2, \ldots$ are actually chosen by the learner depending on the dictionaries built over time, thus building a strong dependency between the stream of points and the dictionary itself. How to analyze such dependency and whether the accuracy of the inducing points is preserved in this case remains as an open question. 
Finally, notice that despite being more elegant and efficient, such incremental dictionary update would not significantly reduce the asymptotic computational complexity, since maximiming $u_t$, whose main cost is computing the posterior variance for each arm, would still dominate the overall runtime.
 
\section{Regret Analysis}
\label{sec:bkb-regret-analysis}
We are now ready to present the second main contribution of this paper,
a bound on the regret achieved by \bkb.
To prove our result we additionally assume that the reward function
$f$ has a bounded norm, i.e.,
$\normsmall{f}_{\rkhs}^2 \triangleq \langle f, f\rangle < \infty$. We use an upper-bound $\normsmall{f}_{\rkhs} \leq F$ to properly tune 
$\wt{\beta}_t$ to the range of the rewards. If $F$ is not known in advance, standard guess-and-double techniques apply.

\begin{theorem}\label{thm:main-regret}
Assume $\normsmall{f}_{\rkhs} \leq F < \infty$.
For any desired $0 < \varepsilon < 1$, $0 < \delta < 1$,
$0 < \lambda$,
let $\alpha \triangleq (1 + \varepsilon)/(1 - \varepsilon)$
and $\qbar \geq 6\alpha\log(4T/\delta)/\varepsilon^2$.
If we run \bkb with\todoaout{What is $\xi$?}
\begin{align*}
\wt{\beta}_t \triangleq
2\xi\sqrt{\alpha\log(\kappa^2t)\left(\sum\nolimits_{s=1}^t \wt{\sigma}_{t}^2(\wt{\bx}_s)\right) + \log(1/\delta)}
+ \left(1 + \tfrac{1}{\sqrt{1-\varepsilon}}\right)\sqrt{\lambda}\fnorm,
\end{align*}
then, with probability of at least $1-\delta$, the regret $R_T$ of \bkb is bounded as
\begin{align*}
R_T \leq 2(2\alpha)^{3/2}\sqrt{T}\left(\xi\deff(\lambda,\wt{\bX}_T)\log(\kappa^2T) +\sqrt{\lambda \fnorm^2 \deff(\lambda,\wt{\bX}_T)\log(\kappa^2T)}+ \xi\log(1/\delta) \right).
\end{align*}\end{theorem}
\vspace{-0.2\baselineskip}
\Cref{thm:main-regret} shows that \bkb achieves exactly the same regret
as (exact) \gpucb up to small $\alpha$ constant and $\log(\kappa^2T)$ multiplicative factor.\footnote{Here we derive a \textit{frequentist} regret bound and thus we compare with the result of \citet{chowdhury2017kernelized} rather than the original \textit{Bayesian} analysis of~\citet{srinivas2010gaussian}.} 
For instance, setting $\varepsilon = 1/2$ results in a bound only $3\log(T)$ times
larger than the one of \gpucb. At the same time, the choice $\varepsilon = 1/2$ only accounts for a constant factor $12$ in the per-step computational complexity, which is still dramatically reduced from $t^2\Narm$ to $\deff^2\Narm$. Note also that even if we send $\varepsilon$ to $0$, in the worst case we will
include all arms selected so far, \ie, $\coldict_t = \{\wt{\bX}_t\}$.
Therefore, even in this case \bkb's runtime does not grow unbounded, but
\bkb transforms back into exact \gpucb.
Moreover, we show that $\deff(\lambda,\wt{\bX}_T) \leq \logdet(\kermatrix_T/\lambda + \bI)$, as in \Cref{prop:logdet-to-deff} in the appendix,
so any bound on $\logdet(\kermatrix_T/\lambda + \bI)$
available for \gpucb applies directly to \bkb. This means that up to an extra $\log T$ factor, we 
match \gpucb's $\abigotime(\log(T)^{2d})$ rate for the Gaussian kernel,
$\abigotime(T^{\frac{1}{2}\frac{2\nu + 3d^2}{2\nu + d^2}})$ rate for the Mat\' ern kernel,
and $\abigotime(d\sqrt{T})$ for the linear kernel. While these bounds are not minimax optimal, they closely follow the lower bounds derived in~\citet{scarlett2017lower}. On the other hand, in the case of linear kernel (i.e., the linear bandits) we nearly match the lower bound of~\citet{dani2008stochastic}. 

Another interesting aspect of \bkb is that computing the trade-off parameter $\wt{\beta}_t$ can be done efficiently. Previous methods bounded this quantity with a loose (deterministic) upper bound, e.g., $\bigotime(\log(T)^d)$ for Gaussian kernels, to avoid the large cost of computing $\logdet(\kermatrix_T/\lambda + \bI)$.
In our~$\wt{\beta}_t$, we bound the $\logdet$ by $\deff$, which is then bounded by $\sum\nolimits_{s=1}^t \wt{\sigma}_{t}^2(\bx_s)$, see Thm.\,\ref{thm:kors-complexity}, where all $\wt{\sigma}_t^2$s are already efficiently computed at each step. While this is up to $\log t$ larger than the exact $\logdet$, it is  \emph{data adaptive} and much smaller than the known worst case upper bounds.

It it crucial, that our regret guarantee is achieved without requiring an
\emph{increasing accuracy} in our approximation. One would expect that to obtain
a sublinear regret the error induced by the approximation should decrease as $1/T$.
Instead, in \bkb, the constants $\varepsilon$ and $\lambda$ that govern the accuracy level are fixed and thus it is not possible to guarantee that $\wt{\mu}_t$ will ever get close to $\mu_t$ everywhere. Adaptivity is the key: we can afford the same approximation level at every step because accuracy is actually increased only on a specific part of the arm set. For example, if a suboptimal arm is selected too often
due to bad approximation, it will be eventually included in $\coldict_t$. After the inclusion, the approximation accuracy in the region of the suboptimal arm increases, and it would not be selected anymore. As the set of inducing points is updated \textit{fast enough}, the impact of inaccurate approximations is limited over time, thus preventing large regret to accumulate.
Note that this is a significant divergence from existing results. In particular
approximation bounds that are uniformly accurate for all $\bx_i \in \armset$,
such as those obtained with quadrature FF \citep{mutny2018efficient},
rely on packing arguments. Due to the nature of packing, this usually causes
the runtime or regret to scale exponentially with the input dimension $d$,
and requires  kernel $\kerfunc$ to have a specific structure, e.g., to be stationary.
Our new analysis avoids both of these problems.

Finally, we point out that the adaptivity of \bkb allows drawing an interesting connection between learning and computational complexity. In fact, both the regret and the computation of \bkb scale with the log-determinant and effective dimension of $\kermatrix_T$, which is related to the effective dimension of the sequence of arms selected over time. As a result, if the problem is difficult from a learning point of view (i.e., the regret is large because of large log-determinant), then \bkb automatically adapts the set $\coldict_t$ by including many more inducing points to guarantee the level of accuracy needed to solve the problem. Conversely, if the problem is simple (i.e., small regret), then \bkb can greatly reduce the size of $\coldict_t$ and achieve the derived level of accuracy.
\\[-0.225in]
\subsection{Proof sketch}\label{sec:approx-ellipsoid-subsec}

We build on the \gpucb analysis of ~\citet{chowdhury2017kernelized}. Their analysis relies on a confidence interval formulation of \gpucb 
that is more conveniently expressed using an explicit feature-based representation of the GP. For any GP with covariance $\kerfunc$, there is a corresponding RKHS~$\rkhs$ with $\kerfunc$ as its kernel function. Furthermore, any kernel function $\kerfunc$ is associated to a non-linear feature map
$\featmap(\cdot): \Real^d \rightarrow \rkhs$
such that $\kerfunc(\bx, \bx') = \featmap(\bx')^\transp\featmap(\bx')$.
As a result, any reward function $f \in \rkhs$ can be written as $f(\bx) = \featmap(\bx)^\transp\bw_\star$, where $\bw_\star \in \rkhs$.

\paragraph{Confidence-interval view of \gpucb.}
Let $\phimat{t} \triangleq [\featmap(\bx_1), \dots, \featmap(\bx_{t})]^\transp$ be the matrix $\bX_t$ after the application of $\featmap(\cdot)$ to each row.
We can then define the regularized design matrix  as $\bA_t \triangleq \phimat{t}^\transp\phimat{t} + \lambda\bI$, and then compute the regularized least-squares estimate as
\vspace{-.35\baselineskip}
\begin{align*}
\wh{\bw}_{t} &\triangleq  \argmin_{\bw \in \rkhs} \sum\nolimits_{i=1}^t (y_i - \phivec{i}^\transp\bw)^2 + \lambda \normsmall{\bw}_2^2
= \bA_t^{-1}\phimat{t}^\transp\by_t.
\end{align*}\\[-\baselineskip]
We define the \emph{confidence interval}
$C_t$ as the ellipsoid induced by $\bA_t$ with center $\wh{\bw}_t$ and radius $\beta_t$
\begin{align}\label{def:oful-confidence-ellipsoid}
C_t &\triangleq \left\{\bw : \normsmall{\bw - \wh{\bw}_{t}}_{\bA_{t}} \leq \beta_{t}\right\}, \quad \quad \beta_t \triangleq \lambda^{1/2}\fnorm + R\sqrt{2(\logdet(\bA_t/\lambda) + \log(1/\delta)},
\end{align}
where the radius $\beta_t$ is such that $\bw_\star \in C_t$ \whp\!\citep{chowdhury2017kernelized}.
Finally, using Lagrange multipliers we reformulate the \gpucb scores as
\vspace{-.35\baselineskip}
\begin{align}\label{eq:ucb-oful}
u_{t}(\bx_i) &= \max_{\bw \in C_t} \phivec{i}^\transp\bw
=\overbracket{\phivec{i}^\transp\wh{\bw}_{t}}^{\mu_t(\bx_i)} + \beta_{t}\overbracket{\sqrt{\phivec{i}^\transp\bA_{t}^{-1}\phivec{i}}.}^{\sigma_t(\bx_i)}
\end{align}

\noindent
\textbf{Approximating the confidence ellipsoid.} 
Consider subset of arm $\coldict_t = \{\bx_i\}_{i=1}^m$
chosen by \bkb at each step
and denote by $\bX_{\coldict_t} \in \Real^{m \times d}$ the matrix with all arms in $\coldict_t$ as rows.
Let  $\wt{\rkhs}_t \triangleq \Ran(\phimat{\coldict_t})$ be
the smaller $m$-rank RKHS spanned by $\phimat{\coldict_t};$
and by $\bP_t$ the symmetric orthogonal projection operator on $\wt{\rkhs}_t$.
We then define an \emph{approximate} feature map
$\wt{\featmap}_t(\cdot) \triangleq \bP_t\featmap(\cdot): \Real^d \rightarrow \wt{\rkhs}_t$
and associated approximations of $\bA_t$ and $\wh{\bw}_t$ as
\begin{align}
\wt{\bA}_t &\triangleq \aphimat{t}^\transp\aphimat{t} + \lambda\bI\label{eq:atilda-def},\\
\abw_t &\triangleq \argmin_{\bw \in \rkhs} \sum_{i=1}^t (y_i - \wt{\featmap}(\bx_i)^\transp\bw)^2 + \lambda \normsmall{\bw}_2^2
=\abA_t^{-1}\aphimat{t}^\transp\by_t\label{eq:wtilda-def}.
\end{align}\\[-\baselineskip]
This leads to an approximate confidence ellipsoid
$\wt{C}_t \triangleq  \big\{\bw : \normsmall{\bw - \abw_{t}}_{\abA_{t}} \leq \wt{\beta}_t\big\}.$ 
A subtle element in these definitions is that while $\aphimat{t}^\transp\aphimat{t}$
and $\abw_t$ are now \emph{restricted} to $\wt{\rkhs}_t$, the identity operator
$\lambda\bI$ in the regularization of $\wt{\bA}_t$ still \emph{acts over the whole} $\rkhs$,
and therefore $\wt{\bA}_t$ does not belong to $\wt{\rkhs}_t$ and remains
full-rank and invertible. This immediately leads to the usage of $k(\bx_i,\bx_i)$ in the definition of $\wt\sigma$ in Eq.\,\ref{eq:approx.gpucb.score}, instead of the its approximate version using the \nystrom embedding.

\textbf{Bounding the regret.}
To find an appropriate $\wt{\beta}_t$ we follow an approach similar to the one of
\citet{abbasi2011improved}. Exploiting the relationship
$y_t = \wt{\featmap}(\wt{\bx}_t)^\transp\bw_\star + \eta_t$,
we bound
\vspace{-.35\baselineskip}
\begin{align*}
\normsmall{\bw_\star - \wt{\bw}_t}^2_{\abA_t}
\leq \overbracket{\lambda^{1/2}\normsmall{\bw_\star}}^{(a)}
+ \overbracket{\normsmall{\aphimat{t}\boldsymbol{\eta}_t}_{\abA_t^{-1}}}^{(b)}
+\overbracket{\normsmall{\phimat{t}^\transp}_{\bI - \bP_t}\cdot\normsmall{\bw_\star}}^{(c)}\!.
\end{align*}
Both $(a)$ and $(b)$ are present in \gpucb and \oful's analysis.
The first term $(a)$ is due to the bias introduced in the least-square estimator $\wt{\bw}_t$
by the regularization $\lambda$.
Then, term $(b)$ is due to the noise in the reward observations. 
Note that the same term $(b)$ appears in \gpucb's analysis
as $\normsmall{\phimat{t}\boldsymbol{\eta}_t}_{\bA_t^{-1}}$
and it is bounded by $\logdet(\bA_t/\lambda)$
using self-normalizing concentration inequalities~\cite{chowdhury2017kernelized}.
However, our $\normsmall{\aphimat{t}\boldsymbol{\eta}_t}_{\abA_t^{-1}}$
is a more complex object, since the projection $\bP_t$ contained
 in $\aphimat{t} \triangleq \bP_t\phimat{t}$ depends on the whole process
up to time time $t$, and therefore $\aphimat{t}$ also depends on the whole
process, losing its martingale structure.
To avoid this, we use Sylvester's identity and the projection operator $\bP_t$ to bound
\vspace{-.35\baselineskip}
\begin{align*}
&\logdet(\abA_t/\lambda)
= \logdet\left(\tfrac{\phimat{t}\bP_t\phimat{t}^\transp}{\lambda} + \bI\right)
\leq \logdet\left(\tfrac{\phimat{t}\phimat{t}^\transp}{\lambda} + \bI\right)
= \logdet(\bA_t/\lambda).
\vspace{-.35\baselineskip}
\end{align*}\\[-1.2\baselineskip]
In other words, restricting the problem to $\wt{\rkhs}_t$ acts as a regularization
and reduces the variance of the martingale.
Unfortunately, $\logdet(\bA_t/\lambda)$ is too expensive to compute,
so we first bound it with $\deff(\lambda, \wt{\bX}_t)\log(\kappa^2t)$,
and then we bound $\deff(\lambda, \wt{\bX}_t) \leq \alpha\sum\nolimits_{s=1}^t \wt{\sigma}_{t}^2(\bx_s),$
\Cref{thm:kors-complexity}, which can be computed efficiently.
Finally, a new bias term $(c)$ appears.
Combining \Cref{thm:kors-complexity} with the results of \citet{calandriello2018statistical}
for  projection $\bP_t$ obtained using RLSs sampling,
we show that
\vspace{-.35\baselineskip}
\begin{align*}
\bI - \bP \preceq \lambda\bA_t^{-1}/(1-\varepsilon).
\end{align*}

\vspace{-.5\baselineskip}
The combination of $(a)$, $(b)$, and $(c)$ leads to the definition of $\wt{\beta}_t$ 
and the final regret bound as
$R_T \leq \sqrt{\wt{\beta}_T}\sqrt{\sum\nolimits_{t=1}^T\phivec{t}^\transp\abA_{t}^{-1}\phivec{t}}$.
To conclude the proof, we bound $\sum\nolimits_{t=1}^T\phivec{t}^\transp\abA_{t}^{-1}\phivec{t}$ with
the following corollary of \Cref{thm:kors-complexity}.

\vspace{-.5\baselineskip}
\begin{corollary}\label[corollary]{cor:A-accuracy-bkb}
Under the same conditions as \Cref{thm:main-regret}, for all $t \in T$, we have 
$\bA_t/\alpha \preceq \wt{\bA}_t \preceq \alpha\bA_t.$
\end{corollary}

\vspace{-.5\baselineskip}
\noindent
\textbf{Remarks.}
The novel bound
$\normsmall{\phimat{t}^\transp}_{\bI - \bP_t} \leq \tfrac{\lambda}{1-\varepsilon}\normsmall{\phimat{t}^\transp}_{\bA_t^{-1}}$ has a crucial role in controlling the bias due to the projection $\bP_t$.
Note that the second term measures the error with the same metric $\bA_t^{-1}$ used by the variance martingale.
In other words, the bias introduced by \bkb's approximation
can be seen as a \emph{self-normalizing} bias. It is larger along directions
that have been sampled less frequently, and smaller along directions
correlated with arms selected often (e.g., the optimal arm).
\\[0.025in]
Our analysis bears some similarity with the one recently and independently
developed by \citet{kuzborskij2019efficient}. 
Nonetheless, our proof improves their result along two dimensions. 
First, we consider the more general (and challenging) GP optimization setting. Second, \emph{we do not fix} the rank of our
approximation in advance. While their analysis also exploits a 
self-normalized bias argument, this applies only to the $k$ largest components. If the problem
has an effective dimension larger than $k$, their radius and regret 
becomes essentially linear. In \bkb we use our adaptive sampling scheme to include
all necessary directions and to achieve the same regret rate as exact \gpucb.

\section{Discussion}\label{sec:discussion}

As the prior work in Bayesian optimization is vast, we do not compare to alternative GP acquisition functions, such as GP-EI or
GP-PI, 
and only focus on approximation techniques with theoretical guarantees.
Similarly, we exclude scalable variational inference based methods,
even when their approximate posterior is provably accurate such as pF-DTC
\citep{huggins2019scalable}, since they only provide guarantees
for GP regression and not for the more difficult  optimization setting.
We also do not discuss \supkernelucb~\citep{valko2013finite}, which has a tighter analysis than \gpucb, since the algorithm
does not work well in practice.

\noindent
\textbf{Infinite arm sets.} Looking at the proof
of \Cref{thm:kors-complexity}, the guarantees on $\wt{u}_{t}$ hold for any $\rkhs$,
and in \Cref{thm:main-regret}, we only require that the maximum $\wt{\bx}_{t+1} \triangleq \argmax_{\bx \in \armset}
\max_{\bw \in \wt{C}_t} \featmap(\bx)^\transp\bw$
is returned. Therefore, the accuracy and regret guarantees also hold also for an infinite set of arms $\armset$.
However, the search over $\armset$ can be difficult. In the general case, maximization
of a GP posterior is an NP-hard problem, with algorithms that often
scale exponentially with the input dimension $d$ and are not practical.
We treated the easier case of finite sets, where enumeration is sufficient. Note that this automatically introduces
an $\Omega(\Narm)$ runtime dependency, which could be removed
if the user  provides an efficient method to solve the maximization problem
on a specific infinite set $\armset$. As an example, \citet{mutny2018efficient}
prove that a GP posterior approximated using QFF can be optimized efficiently
in low dimensions and we expect similar results hold for \bkb and low
\emph{effective} dimension. Finally, note that recomputing a new set $\coldict_t$ still requires
$\min\{\Narm, t\}\deff^2$ at each step. As discussed at the end of \Cref{sec:bkb}, this is a bottleneck in $\bkb$ due to the
non-incremental dictionary sampling and independent
from the arm selection. How to address it remains an open question.\\
\noindent
\textbf{Linear bandit with matrix sketching.}
Our analysis is  related to the ones of \textsc{CBRAP}~ \citep{yu2017contextual} and \textsc{SOFUL}~\citep{kuzborskij2019efficient}.
\textsc{CBRAP} uses Gaussian projections to embed all arms in a lower dimensional
space for efficiency. Unfortunately their approach must either use
an embedded space at least $\Omega(T)$ large, which in most cases would
be even slower than exact \oful, or it incurs linear regret \whp
Another approach for Euclidean spaces based on matrix approximation is \soful, introduced by~\citet{kuzborskij2019efficient}.
It uses Frequent Direction~\citep{ghashami2016frequent}, a method similar to
incremental PCA, to embed the arms into $\Real^m$, where $m$ is \emph{fixed} in advance.
To compare, we  distinguish between \textsc{SOFUL-UCB} and \textsc{SOFUL-TS},
a variant based on Thompson sampling. \textsc{SOFUL-UCB} achieves a $\abigotime(T\Narm m^2)$ runtime and
$\abigotime((1 + \varepsilon_m)^{3/2}(d + m)\sqrt{T})$ regret,
where $\varepsilon_m$ is the sum of the $d-m$ smallest eigenvalues
of $\bA_T$. However, notice that if the tail do not decrease quickly,
this algorithm also suffers linear regret and no adaptive way to tune $m$ is known.
On the same task \bkb  achieves
a $\abigotime(d\sqrt{T})$ regret, since it adaptively chooses the size of the embedding.
Computationally, directly instantiating \bkb to use
a linear kernel would achieve a $\abigotime(T\Narm m_t^2)$ runtime\footnote{Note that for both algorithms the bottleneck is maximizing the UCB.}, matching
\citet{kuzborskij2019efficient}'s.
Compared to \textsc{SOFUL-TS}, \bkb achieves better regret,
but is potentially slower. Since Thompson sampling
does not need to compute all confidence intervals, but  solves a simpler
optimization problem, \textsc{SOFUL-TS}
requires only $\abigotime(T\Narm m)$ time against \bkb's $\abigotime(T\Narm m_t^2)$.
It is unknown if a variant
of \bkb can match this complexity.

\textbf{Approximate GP with RFF.} Traditionally, RFF approaches have been popular
to transform GP optimization in a finite-dimensional problem and allow for
scalability. Unfortunately \gpucb with traditional RFF is not low-regret, as RFF are
well known to suffer from
variance starvation~\citep{wang2018batched} and unfeasibly large RFF embeddings
would be necessary to prevent it.
Recently, \citet{mutny2018efficient} proposed an alternative approach
based on QFF, a specialized approach to random features
for stationary kernels.
They achieve the same regret rate as \gpucb and \bkb, with a near-optimal $\bigotime(T\Narm\log(T)^{d+1})$
runtime. Moreover they present an additional variations based on Thompson sampling
whose posterior can be exactly maximized in polynomial time if the input data
is low dimensional or the covariance $\kerfunc$ additive, while it is still an open question
how to efficiently maximize \bkb's UCB $\wt{u}_t$ for infinite $\armset$.
However QFF based approaches apply to stationary kernel only, and require to $\varepsilon$-cover $\armset$,
hence they cannot escape an exponential dependency on the dimensionality $d$.
Conversely \bkb can be applied to any kernel function, and
while not specifically designed for this task it also achieve
a close $\abigotime(T\Narm\log(T)^{3(d+1)})$ runtime.
Moreover, in practice the size of $\coldict_T$ is
less than exponential in~$d$.\todoaout{How does it compare to them? Done, Daniele}

\vskip 0.5cm
{\small \noindent 
\textbf{Acknowledgements\ \ }
This material is based upon work supported by the Center for Brains, Minds and Machines (CBMM), funded by NSF STC award CCF-1231216. L. R. acknowledges the financial support of the AFOSR projects FA9550-17-1-0390  and BAA-AFRL-AFOSR-2016-0007 (European Office of Aerospace Research and Development), and the EU H2020-MSCA-RISE project NoMADS - DLV-777826.
The research presented was also supported by European CHIST-ERA project DELTA, French Ministry of Higher Education and Research, Nord-Pas-de-Calais Regional Council, Inria and Otto-von-Guericke-Universit\"at Magdeburg associated-team north-European project Allocate, and French National Research Agency project BoB (n.ANR-16-CE23-0003).
This research has also benefited from the support of the FMJH Program PGMO and from the support to this program from Criteo.
}

\newpage
\onecolumn
\appendix

\section{Relaxing assumptions}
In our derivations, we make several assumptions. While some are
necessary, others can be relaxed.

\noindent
\textbf{Assumptions on the noise.}
Throughout the paper, we assume that the noise $\eta_t$ is i.i.d.\,Gaussian.
Since \citeauthor{chowdhury2017kernelized}'s results hold for any $\xi$-sub-Gussian
noise that is measurable based with respect to the prior observations, this assumption can
be easily relaxed. 
\todomi{adaptation
    to noise - if there is no noise, we can learn
    exponentially fast (for simple regret and therefore bounded cumulative one)
    https://www.icml.cc/2012/papers/853.pdf
    and we can adapt to it: https://arxiv.org/abs/1810.00997
}

\noindent
\textbf{Assumptions on the arms.}
So far we considered a set of arms that is $(a)$ in $\Real^d$, $(b)$ fixed for all~$t$, and $(c)$ finite.
Relaxing $(a)$ is easy, since we do not make any assumption beyond boundedness
on the kernel function $\kerfunc$ and there are many bounded kernel function
for non-Euclidean spaces, e.g., strings or graphs.
Relaxing $(b)$ is trivial, we just need to embed the changing arm sets
as they are provided, and store and re-embed previously selected arms as necessary.
The per-step time complexity will now depend on the size of the set of arms
available at each step.
Relaxing $(c)$ is straightforward from a theoretical perspective, but has
varying computational consequences. In particular, looking at the proof
of \Cref{thm:kors-complexity}, the guarantees on $\wt{u}_{t}$ hold for all $\rkhs$
and in \Cref{thm:main-regret}, we only require that the maximum $\wt{\bx}_{t+1} \triangleq \argmax_{\bx \in \armset}
\max_{\bw \in \wt{C}_t} \featmap(\bx)^\transp\bw$
is returned. Therefore, at least from the regret point of view, everything
holds also for infinite $\armset$.
However, while the inner maximization over $\wt{C}_t$ can be solved in closed
form for a fixed $\bx$, the same cannot be said of the search over $\armset$.
If the designer can provide an efficient method to solve the maximization problem
on an infinite $\armset$, \eg, linear bandit optimization over compact subsets or $\Real^d$,
then  all \bkb guarantees apply.

\section{Properties of the posterior variance}
For simplicity and completeness we provide  known statements regarding the posterior
variance $\sigma_t^2(\cdot)$. While most of these hold for generic RLS, we will
adapt them to our notation.

\begin{proposition}[\citealp{calandriello2017distributed}]\label[proposition]{prop:tau-decrease}
For the posterior variance, we have that
\begin{align*}
\frac{1}{\kappa^2/\lambda + 1}\sigma_{t-1}^2(\wt{\bx}_t)
\leq \frac{1}{\sigma_{t-1}^2(\wt{\bx}_t) + 1}\sigma_{t-1}^2(\wt{\bx}_t)
\leq \sigma_t^2(\wt{\bx}_t)
\leq \sigma_{t-1}^2(\wt{\bx}_t).
    \end{align*}
\end{proposition}
\begin{proof}
The leftmost inequality follows from $\kappa^2/\lambda \geq \sigma_0^2(x)$  and $\sigma^2_a(x) \geq \sigma^2_b(x), \forall a \leq b$, the others are 
are by \citealp{calandriello2017distributed}.
\end{proof}

\begin{proposition}[\citealp{hazan2006logarithmic,calandriello2017second}]\label[proposition]{prop:logdet-to-deff}
The effective dimension $\deff(\lambda, \wt{\bX}_T)$ is upperbounded as
\begin{align*}
\deff(\lambda, \wt{\bX}_T) &\triangleq \Tr(\kermatrix_T(\kermatrix_T + \lambda\bI)^{-1}) = \sum\nolimits_{t=1}^T \sigma_T^2(\wt{\bx}_t)\\
&\stackrel{(1)}{\leq} \sum\nolimits_{t=1}^T \sigma_t^2(\wt{\bx}_t)\\
&\stackrel{(2)}{\leq} \logdet\left(\kermatrix_T/\lambda + \bI\right)\\
&\stackrel{(3)}{\leq} \Tr(\kermatrix_T(\kermatrix_T + \lambda\bI)^{-1})\left(1 +\log\left(\tfrac{\normsmall{\kermatrix_T}}{\lambda} + 1\right)\right).
\end{align*}
\end{proposition}
\begin{proof}
Inequality $(1)$ is due to \Cref{prop:tau-decrease}, inequality $(2)$ is due to \citet{hazan2006logarithmic}, and
Inequality $(3)$ is due to \citet{calandriello2017second}.
\end{proof}

\section{Proof of \Cref{thm:kors-complexity}}\label{sec:app-proof-kors}
Let $B_t$ be the unfavorable event where the guarantees of \Cref{thm:kors-complexity}
do not hold. Our goal is to prove that $B_t$ happens at most
with probability $\delta$ uniformly for all $t \in [T]$.

\subsection{Notation}
In the following we refer to
$\boldsymbol{\Phi}(\wt{\bX}_t)$ as $\boldsymbol{\Phi}_t$, $\wt{\boldsymbol{\Phi}}(\wt{\bX}_t)$ as~$\wt{\boldsymbol{\Phi}}_t$
and $\boldsymbol{\phi}(\wt{\bx}_t)$ as $\boldsymbol{\phi}_t$. When the subscript is clear from the context,
we omit it.
Since we leverage several results of \citet{calandriello2017second},
we start with some  additional notation.
\renewcommand{\phimat}{\boldsymbol{\Phi}}
\renewcommand{\aphimat}{\wt{\boldsymbol{\Phi}}}
\renewcommand{\phivec}{\boldsymbol{\phi}}
\newcommand{\aphivec}{\wt{\boldsymbol{\phi}}}

First we extend our notation for the subset $\coldict_t$ to include a possible reweighing of the inducing points.
We denote with $\coldict_t \triangleq \{(\phivec_{j}, s_{j})\}_{j=1}^{m_t},$
a \emph{weighted} subset, i.e., a weighted \emph{dictionary}, of columns from $\phimat_t$,
with positive weights $s_j > 0$ that must be appropriately chosen.
Now, denote with $i_j \in [t]$, the index
of the sample $\phivec_j$ as a column in $\phimat_t$.
Using a standard approach \citep{alaoui2014fast}, we choose $s_{j} \triangleq 1/\sqrt{\wt{p}_{t,i_j}}$,
where $\wt{p}_{t,i} \triangleq \qbar\wt{\sigma}_{t-1}^2(\wt{\bx}_i)$ is the probability\footnote{Note that $\wt{p}_{t,i}$ might be larger than 1, but with a small abuse of notation and without the loss of generality we  still refer to it as a probability.} used
by \Cref{alg:bkb} when sampling $\phivec_{i_j}$ from $\phimat_t$.

Let  $\selmatrix_t \in \Real^{t \times t}$ be the diagonal
matrix with $q_{t,i}/\sqrt{\wt{p}_{t,i}}$ on the diagonal,
where $q_{t,i}$ are the $\{0,1\}$ random variables selected by \Cref{alg:bkb}.
Then, we can see that
\begin{align}\label{eq:def-dict-as-sum}
\sum_{j=1}^{m_t} \frac{1}{\wt{p}_{t,i_j}}\phivec_{i_j}\phivec_{i_j}^\transp
= \sum_{i=1}^{t} \frac{q_{t,i}}{\wt{p}_{t,i}}\phivec_{i}\phivec_{i}^\transp
= \phimat_t\selmatrix_t\selmatrix_t^\transp\phimat_t^\transp.
\end{align}
\citet{calandriello2017distributed} define $\coldict_t$ to be an $\varepsilon$-accurate dictionary of $\phimat_t$
if it satisfies  
\begin{align}\label{eq:def-eps-lambda-accurate}
(1-\varepsilon)\phimat_t\phimat_t^\transp - \varepsilon\lambda\bI \preceq \phimat_t\selmatrix_t\selmatrix_t^\transp\phimat_t^\transp \preceq (1+\varepsilon)\phimat_t\phimat_t^\transp + \varepsilon\lambda\bI.
\end{align}
\todom{missing ref}
We can also now fully define the projection operator at time $t$ (see \Cref{sec:approx-ellipsoid-subsec} for more details) as
\begin{align*}
\bP_t \triangleq \phimat_t\selmatrix_t(\selmatrix_t^\transp\phimat_t^\transp\phimat_t\selmatrix_t)^{+}\selmatrix_t^\transp\phimat_t^\transp, 
\end{align*}
which is the projection matrix spanned by the dictionary.

\subsection{Event decomposition}
We  decompose \Cref{thm:kors-complexity} into an accuracy part, i.e., $\coldict_t$ must induce accurate $\wt{\sigma}_t$,
and an efficiency part, i.e., $m_t \leq \deff(t)$.
We also  the accuracy of $\wt{\sigma}_t$
to the definition of $\varepsilon$-accuracy.
\begin{lemma}\label[lemma]{lem:A-accurate}
Let $\alpha \triangleq \frac{1+\varepsilon}{1-\varepsilon}$. If $\coldict_t$ is $\varepsilon$-accurate w.r.t.\,$\phimat_t$, then
\[
\bA_t/\alpha \preceq \abA_t \preceq \alpha \bA_t
\quad \text{and} \quad
\sigma_t^2(\bx)/\alpha \leq \min\left\{\wt{\sigma}_t^2(\bx), 1\right\} \leq \alpha\sigma_t^2(\bx)\; \text{\ for all\ } \bx \in \armset.
\]
\end{lemma}
\begin{proof}
Inverting the bound in \Cref{eq:def-eps-lambda-accurate} and using the fact that $\bP_t\phimat_t\selmatrix_t = \phimat_t\selmatrix_t,$ we get
\begin{align*}
	\bP_t\phimat_t\phimat_t^\transp\bP_t
&\preceq \frac{1}{1-\varepsilon}(\bP_t\phimat_t\selmatrix_t\selmatrix_t^\transp\phimat_t^\transp\bP_t + \varepsilon\lambda\bP_t)
	\preceq \frac{1}{1-\varepsilon}(\phimat_t\selmatrix_t\selmatrix_t^\transp\phimat_t^\transp + \varepsilon\lambda\bP_t)\\
		&\preceq \frac{1}{1-\varepsilon}((1+\varepsilon)\phimat_t\phimat_t^\transp + \varepsilon\lambda\bI + \varepsilon\lambda\bP_t)
		\preceq \frac{1+\varepsilon}{1-\varepsilon}\left(\phimat_t\phimat_t^\transp + \frac{2\varepsilon}{1+\varepsilon}\lambda\bI\right).
		\end{align*}
Repeating the same process for the other side, we obtain
\begin{align*}
\frac{1-\varepsilon}{1+\varepsilon}\left(\phimat_t\phimat_t^\transp - \frac{2\varepsilon}{1-\varepsilon}\lambda\bI\right)
\preceq \bP_t\phimat_t\phimat_t^\transp\bP_t
\preceq \frac{1+\varepsilon}{1-\varepsilon}\left(\phimat_t\phimat_t^\transp + \frac{2\varepsilon}{1+\varepsilon}\lambda\bI\right).
\end{align*}
Applying the above to $\abA_t,$ we get
\begin{align*}
\abA_t = \bP_t\phimat_t\phimat_t^\transp\bP_t + \lambda\bI
\succeq \frac{1-\varepsilon}{1+\varepsilon}\left(\phimat_t\phimat_t^\transp - \frac{2\varepsilon}{1-\varepsilon}\lambda\bI\right)
+ \lambda\bI
= \frac{1-\varepsilon}{1+\varepsilon}\left(\phimat_t\phimat_t^\transp +\lambda\bI\right)
= \frac{1-\varepsilon}{1+\varepsilon}\bA_t,
\end{align*}
which can again be applied on the other side to obtain our result.
To prove the accuracy of the approximate posterior variance $\wt{\sigma}_t^2(\bx_i)$ we simply
apply the definition to get
\[
\frac{1-\varepsilon}{1+\varepsilon}\overbracket{\phivec_i^\transp\bA_t\phivec_i}^{{\sigma}_t^2(\bx_i)}
\preceq \overbracket{\phivec_i^\transp\abA_t\phivec_i}^{\wt{\sigma}_t^2(\bx_i)}
\preceq \frac{1+\varepsilon}{1-\varepsilon}\overbracket{\phivec_i^\transp\bA_t\phivec_i}^{{\sigma}_t^2(\bx_i)}.
\]
\end{proof}
Using \Cref{lem:A-accurate}, we decompose our unfavorable event $B_t \triangleq A_t \cup E_t,$
where $A_t$ is the event where~$\coldict_t$ is not $\varepsilon$-accurate w.r.t.\,$\phimat_t$
and $E_t$ is the event where $m_t$ is much larger than $\deff(\lambda, \wt{\bX}_t)$.
We now further decompose the event $A_t$ as
\begin{align*}
A_t &= (A_t \cap A_{t-1}) \cup (A_t \cap A_{t-1}^\comp)\\
&\subseteq
A_{t-1} \cup (A_t \cap A_{t-1}^\comp)
= A_0 \cup \left(\bigcup_{s=1}^t (A_s \cap A_{s-1}^\comp)\right)
= \bigcup_{s=1}^t (A_s \cap A_{s-1}^\comp),
\end{align*}
where $A_0$ is the empty event since $\phimat_0$ is empty and it is well
approximated by the empty $\coldict_0$.
Moreover, we  simplify a part of the expression by noting
\begin{align*}
B_t = A_t \cup E_t = A_t \cup (E_t \cap A_{t-1}^\comp) \cup (E_t \cap A_{t-1})
\subseteq A_t \cup A_{t-1} \cup (E_t \cap A_{t-1}^\comp),
\end{align*}
which will help us
when bounding the event $E_t$, where we will directly act as if $A_t$ does not hold.
Putting it all together, we get
\begin{align*}
\bigcup_{t=1}^T B_t
&= \bigcup_{t=1}^T (A_t \cup E_t)
\subseteq \bigcup_{t=1}^T \left(A_t \cup A_{t-1} \cup (E_t\cap A_{t-1}^{\comp})\right)\\
&= \left(\bigcup_{t=1}^T A_t\right) \cup \left(\bigcup_{t=1}^T (E_t\cap A_{t-1}^{\comp})\right)
= \left(\bigcup_{t=1}^T A_t\right) \cup \left(\bigcup_{t=1}^T (E_t\cap A_{t-1}^{\comp})\right)\\
&\subseteq \left(\bigcup_{t=1}^T \left(\bigcup_{s=1}^t (A_s \cap A_{s-1}^\comp)\right)\right) \cup \left(\bigcup_{t=1}^T (E_t\cap A_{t-1}^{\comp})\right)\\
&= \left(\bigcup_{t=1}^T (A_t \cap A_{t-1}^\comp)\right) \cup \left(\bigcup_{t=1}^T (E_t\cap A_{t-1}^{\comp})\right).
\end{align*}

\subsection{Bounding $\Pr(A_t \cap A_{t-1}^\comp)$}
We now bound the probability of event $A_t \cap A_{t-1}^\comp$.
In our first step, we formally define $A_t$ using \Cref{eq:def-eps-lambda-accurate}.
In particular, we rewrite the $\varepsilon$-accuracy condition as
\begin{align*}
&(1-\varepsilon)\phimat_t\phimat_t^\transp - \varepsilon\lambda\bI \preceq \phimat_t\selmatrix_t\selmatrix_t^\transp\phimat_t^\transp \preceq (1+\varepsilon)\phimat_t\phimat_t^\transp + \varepsilon\lambda\bI\\
&\iff - \varepsilon(\phimat_t\phimat_t^\transp + \lambda\bI) \preceq \phimat_t\selmatrix_t\selmatrix_t^\transp\phimat_t^\transp - \phimat_t\phimat_t^\transp \preceq \varepsilon(\phimat_t\phimat_t^\transp + \lambda\bI)\\
&\iff - \varepsilon\bI \preceq (\phimat_t\phimat_t^\transp + \lambda\bI)^{-1/2}(\phimat_t\selmatrix_t\selmatrix_t^\transp\phimat_t^\transp - \phimat_t\phimat_t^\transp)(\phimat_t\phimat_t^\transp + \lambda\bI)^{-1/2} \preceq \varepsilon\bI\\
&\iff \normsmall{(\phimat_t\phimat_t^\transp + \lambda\bI)^{-1/2}(\phimat_t\selmatrix_t\selmatrix_t^\transp\phimat_t^\transp - \phimat_t\phimat_t^\transp)(\phimat_t\phimat_t^\transp + \lambda\bI)^{-1/2}} \leq \varepsilon,
\end{align*}
where $\normsmall{\cdot}$ is the spectral norm.
We  now focus on the last reformulation and frame it as a random matrix concentration
question in RKHS $\rkhs$. Let $\pvec_{t,i} \triangleq (\phimat_t\phimat_t^\transp + \lambda\bI)^{-\tfrac{1}{2}}\phivec_i$
and $\pmat_t \triangleq \phimat_t(\phimat_t^\transp\phimat_t + \lambda\bI)^{-\tfrac{1}{2}} =[\pvec_{t,1}, \dots, \pvec_{t,t}]^\transp,$
and define the operator $\bG_{t,i} \triangleq \left(\frac{q_{t,i}}{\wt{p}_{t,i}} - 1\right)\psi_{t,i}\psi_{t,i}^\transp$.
Then we rewrite $\varepsilon$-accuracy as
\begin{align*}
&\normempty{(\phimat_t\phimat_t^\transp\! +\! \lambda\bI)^{-\tfrac{1}{2}}\phimat_t(\selmatrix_t\selmatrix_t^\transp \!\!-\!\! \bI)\phimat_t^\transp(\phimat_t\phimat_t^\transp\! +\! \lambda\bI)^{-\tfrac{1}{2}}}
= \normempty{\sum_{i=1}^t\left(\frac{q_{t,i}}{\wt{p}_{t,i}} - 1\right)\pvec_{t,i}\pvec_{t,i}^\transp}
= \normempty{\sum_{i=1}^t\bG_{t,i}} \leq \varepsilon,
\end{align*}
and the event $A_t$ as the event where $\normempty{\sum_{i=1}^t\bG_{t,i}} \geq \varepsilon$,
Note that this reformulation
exploits the fact that $q_{t,i} = 0$ encodes the column that are not selected
in $\coldict_t$ (see \Cref{eq:def-dict-as-sum}).
To study this random object, we begin by defining the filtration $\filtration_t \triangleq \{q_{s,i}, \eta_s\}_{s=1}^t$ at time $t$
containing all the randomness coming from the construction of the various
$\coldict_s$ and the noise on the function $\eta_t$.
In particular, note that the $\{0,1\}$ r.v.\@ $q_{t,i}$ used by \Cref{alg:bkb}
are not necessarily Bernoulli r.v.s, since the probability $\wt{p}_{t,i}$ used to select $0$
or $1$ is itself random. However, they become well defined Bernoulli when conditioned
on $\filtration_{t-1}$.
Let $\indfunc\{\cdot\}$ indicates the indicator function of an event. We have that
\begin{align*}
\Pr(A_t \cap A_{t-1}^\comp)
&= \Pr\left(\normempty{\sum_{i=1}^t\bG_{t,i}} \geq \varepsilon \cap \normempty{\sum_{i=1}^t\bG_{t-1,i}} \leq \varepsilon\right)\\
&= \expectedvalue_{\filtration_t}\left[ \indfunc\left\{\normempty{\sum_{i=1}^t\bG_{t,i}} \geq \varepsilon \cap \normempty{\sum_{i=1}^t\bG_{t-1,i}} \leq \varepsilon\right\}\right]\\
&= \expectedvalue_{\filtration_{t-1}}\left[\expectedvalue_{\eta_t, \{q_{t,i}\}}\left[ \indfunc\left\{\normempty{\sum_{i=1}^t\bG_{t,i}} \geq \varepsilon \cap \normempty{\sum_{i=1}^t\bG_{t-1,i}} \leq \varepsilon\right\}\condbar \filtration_{t-1}\right]\right]\\
&= \expectedvalue_{\filtration_{t-1}}\left[\expectedvalue_{\{q_{t,i}\}}\left[ \indfunc\left\{\normempty{\sum_{i=1}^t\bG_{t,i}} \geq \varepsilon \cap \normempty{\sum_{i=1}^t\bG_{t-1,i}} \leq \varepsilon\right\}\condbar \filtration_{t-1}\right]\right],
\end{align*}
where the last passage is due to the fact that $\bG_{t,i}$ is independent from
$\eta_t$. Next, notice that conditioned on $\filtration_{t-1},$ the event $A_{t-1}^{\comp}$ becomes 
deterministic, and we can restrict our expectations to the outcomes
where $\normempty{\sum_{i=1}^t\bG_{t-1,i}} \leq \varepsilon$,
\begin{align*}
\Pr(A_t \cap A_{t-1}^\comp)
&= \expectedvalue_{\filtration_{t-1}: \normempty{\sum_{i=1}^t\bG_{t-1,i}} \leq \varepsilon}\left[\expectedvalue_{\{q_{t,i}\}}\left[ \indfunc\left\{\normempty{\sum_{i=1}^t\bG_{t,i}} \geq \varepsilon \right\}\condbar \filtration_{t-1} \right]\right].
\end{align*}
Moreover, conditioned on $\filtration_{t-1}$ all the $q_{t,i}s$ become independent
r.v., and we are able to use the following result of \citet{tropp2015an-introduction}.
\begin{proposition}\label[proposition]{prop:matrix-freedman}
Let $\bG_1,\dots,\bG_n$ be a sequence of independent self-adjoint random operators such that $\expectedvalue\left[\bG_i\right] = 0$ and $\normempty{\bG_i} \leq R$ a.s. Denote $\sigma^2 = \normempty{\sum_{i=1}^t \expectedvalue \left[ \bG_i^2 \right] }$. Then, for any $\epsilon \geq 0,$
\begin{align*}
\Pr\left(\normempty{\sum_{i=1}^t \bG_i}\geq \epsilon \right) \leq 4 t \exp\left(\frac{\epsilon^2/2}{\sigma^2 +R\epsilon/3}\right)\cdot
\end{align*} 
\end{proposition}
We begin by computing the mean of $\bG_{t,i},$
\begin{align*}
\expectedvalue_{q_{t,i}}\left[\bG_{t,i}\condbar \filtration_{t-1}\right]
&= \expectedvalue_{q_{t,i}}\left[\left(\frac{{q}_{t,i}}{\wt{p}_{t,i}} - 1\right)\pvec_{t,i}\pvec_{t,i}^\transp\condbar \filtration_{t-1}\right]\\
&= \left(\frac{\expectedvalue_{q_{t,i}}\left[{q}_{t,i}\condbar \filtration_{t-1}\right]}{\wt{p}_{t,i}} - 1\right)\pvec_{t,i}\pvec_{t,i}^\transp
= \left(\frac{\wt{p}_{t,i}}{\wt{p}_{t,i}} - 1\right)\pvec_{t,i}\pvec_{t,i}^\transp = \bsym{0},
\end{align*}
where we use the fact that $\wt{p}_{t,i}$ is fixed conditioned on $\filtration_{t-1}$ and it is the (conditional) expectation of~${q}_{t,i}$.
Since $\bG$ is zero-mean, we can use \Cref{prop:matrix-freedman}.
First, we find $R$ and for that, we upper bound
\begin{align*}
\normempty{\bG_{t,i}}
=\normempty{\left(\frac{{q}_{t,i}}{\wt{p}_{t,i}} - 1\right)\pvec_{t,i}\pvec_{t,i}^\transp}
\leq \left|\left(\frac{{q}_{t,i}}{\wt{p}_{t,i}} - 1\right)\right|\normsmall{\pvec_{t,i}\pvec_{t,i}^\transp}
\leq \frac{1}{\wt{p}_{t,i}}\normsmall{\pvec_{t,i}\pvec_{t,i}^\transp}.
\end{align*}
Note that due to the definition of $\pvec_{t,i},$
\begin{align*}
\normsmall{\pvec_{t,i}\pvec_{t,i}^\transp}
= \pvec_{t,i}^\transp\pvec_{t,i}
= \phivec_i^\transp(\phimat_t\phimat_t^\transp + \lambda\bI)^{-1}\phivec_i
= \sigma_t^2(\wt{\bx}_i).
\end{align*}
Moreover, we are only considering outcomes of $\filtration_{t-1}$ where $\normempty{\sum_{i=1}^t\bG_{t-1,i}} \leq \varepsilon$, which implies that $\coldict_{t-1}$ is $\varepsilon$-accurate,
and by \Cref{lem:A-accurate} we have that $\wt{\sigma}_{t-1}(\wt{\bx}_i) \geq \sigma_{t-1}(\wt{\bx}_i)/\alpha$.
Finally, due to \Cref{prop:tau-decrease}, we have
$\sigma_{t-1}(\wt{\bx}_i) \geq \sigma_t(\wt{\bx}_i)$.
Putting this all together we can bound
\begin{align*}
\frac{1}{\wt{p}_{t,i}}\normsmall{\pvec_{t,i}\pvec_{t,i}^\transp}
= \frac{1}{\qbar\wt{\sigma}_{t-1}(\wt{\bx}_i)}{\sigma}_{t}(\wt{\bx}_i)
\leq \frac{\alpha}{\qbar}
\triangleq R.
\end{align*}
For the variance term, we expand
\begin{align*}
\sum_{i=1}^t \expectedvalue_{q_{t,i}}\left[\bG_{t,i}^2 \condbar \filtration_{t-1}\right]
&= \sum_{i=1}^t \expectedvalue_{q_{t,i}}\left[\left(\frac{{q}_{t,i}}{\wt{p}_{t,i}} - 1\right)^2 \condbar \filtration_{t-1}\right] \pvec_{t,i}\pvec_{t,i}^\transp\pvec_{t,i}\pvec_{t,i}^\transp\\
&= \sum_{i=1}^t
\left(\expectedvalue_{q_{t,i}}\left[\frac{{q}_{t,i}^2}{\wt{p}_{t,i}^2} \condbar \filtration_{t-1}\right]
- \expectedvalue_{q_{t,i}}\left[2\frac{{q}_{t,i}}{\wt{p}_{t,i}} \condbar \filtration_{t-1}\right]
+ 1\right) \pvec_{t,i}\pvec_{t,i}^\transp\pvec_{t,i}\pvec_{t,i}^\transp\\
&= \sum_{i=1}^t
\left(\expectedvalue_{q_{t,i}}\left[\frac{{q}_{t,i}}{\wt{p}_{t,i}^2} \condbar \filtration_{t-1}\right] - 1\right) \pvec_{t,i}\pvec_{t,i}^\transp\pvec_{t,i}\pvec_{t,i}^\transp
=\sum_{i=1}^t\left(\expectedvalue_{q_{t,i}}\left[\frac{{q}_{t,i}}{\wt{p}_{t,i}^2} \condbar \filtration_{t-1}\right] - 1\right) \pvec_{t,i}\pvec_{t,i}^\transp\pvec_{t,i}\pvec_{t,i}^\transp\\
&=\sum_{i=1}^t\left(\frac{1}{\wt{p}_{t,i}} - 1\right) \pvec_{t,i}\pvec_{t,i}^\transp\pvec_{t,i}\pvec_{t,i}^\transp
\preceq\sum_{i=1}^t\frac{1}{\wt{p}_{t,i}} \normsmall{\pvec_{t,i}\pvec_{t,i}^\transp}\pvec_{t,i}\pvec_{t,i}^\transp
\preceq\sum_{i=1}^tR\pvec_{t,i}\pvec_{t,i}^\transp,
\end{align*}
where we used the fact that ${q}_{t,i}^2 = {q}_{t,i}$ and $\expectedvalue_{q_{t,i}}[{q}_{t,i} | \filtration_{t-1}] = \wt{p}_{t,i}$.
We can now bound this quantity as
\begin{align*}
\normempty{\sum_{i=1}^t \expectedvalue_{q_{t,i}}\left[\bG_{t,i}^2 \condbar \filtration_{t-1}\right]}
\leq \normempty{\sum_{i=1}^tR\pvec_{t,i}\pvec_{t,i}^\transp}
= R\normempty{\sum_{i=1}^t\pvec_{t,i}\pvec_{t,i}^\transp}
= R\normsmall{\pmat_{t}^\transp\pmat_{t}}
\leq R \triangleq \sigma^2.
\end{align*}
Therefore, we have $\sigma^2 = R$ and $R = 1/\qbar$. Now, applying \Cref{prop:matrix-freedman}
and a union bound we conclude the proof.

\subsection{Bounding $\Pr(E_t \cap A_{t-1}^\comp)$}
We will use the following concentration for independent Bernoulli random variables.
\begin{proposition}[\citealp{calandriello2017distributed}, App.\,D.4]\label[proposition]{prop:hoeff-sum-bern}
Let $\{q_s\}_{s=1}^t$ be independent
Bernoulli random variables, each with success probability $p_s$,
and let $d = \sum_{s=1}^t p_s \geq 1$ be their sum. Then,\footnote{This is a simple variant of the Chernoff bound where the Bernoulli random variables are not identically distributed.}
\begin{align*}
\probability\left(\sum_{s=1}^t q_s \geq 3d\right) \leq \exp\{-3d(3d - (\log(3d)+1))\} \leq \exp\{-2d\}.
\end{align*}
\end{proposition}
We now rigorously define  event $E_t$ as the event
where
\begin{align*}
\sum_{i=1}^t q_{t,i} \geq 3\alpha(1+\kappa^2/\lambda)\log(t/\delta)\sum_{i=1}^t \sigma_t^2(\wt{\bx}_i) = 3\alpha(1+\kappa^2/\lambda)\deff(\lambda,\wt{\bX}_t)\log(t/\delta).
\end{align*}
Once again, we use conditioning and in particular,
\begin{align*}
\!\Pr(E_t \! \cap \! A_{t}^\comp)
&= \!\!\!\expectedvalue_{\filtration_{t-1}: \normempty{\sum_{i=1}^t\!\bG_{t-1,i}} \leq \varepsilon}\left[\expectedvalue_{\{q_{t,i}\}}\left[ \indfunc\left\{ \sum_{i=1}^t q_{t,i} \!\geq\! 3\alpha(1\!+\!\kappa^2/\lambda)\!\log(t/\delta)\sum_{i=1}^t \sigma_t^2(\wt{\bx}_i)\!\right\}\!\!\condbar\! \filtration_{t-1}\! \right]\!\right]\!.
\end{align*}
Conditioned on $\filtration_{t-1}$ the r.v.\,$q_{t,i}$ becomes independent Bernoulli
with probability $\wt{p}_{t,i} \triangleq \qbar \wt{\sigma}_{t-1}(\wt{\bx}_i)$.
Since we restrict the outcomes to $A_{t-1}^{\comp}$, we can exploit
\Cref{lem:A-accurate} and the guarantees of $\varepsilon$-accuracy to bound
$\wt{p}_{t,i} \leq \alpha\sigma_{t-1}^2(\wt{\bx}_i)$. Then, we  use
\Cref{prop:tau-decrease} to bound $\sigma_{t-1}^2(\wt{\bx}_i) \leq (1 + \kappa^2/\lambda)\sigma_t^2(\wt{\bx}_i)$.
Therefore, $q_{t,i}$ are conditionally independent Bernoulli with probability
at most $\qbar(1 + \kappa^2/\lambda)\sigma_t^2(\wt{\bx}_i)$.
Applying a simple stochastic dominance argument and \Cref{prop:hoeff-sum-bern}
gets the needed statement.
 
\section{Proof of \Cref{thm:main-regret}}
Following \citet{abbasi2011improved}, we divide the proof in two parts, first bounding the approximate confidence ellipsoid,
and then bounding the regret.

\subsection{Bounding the confidence ellipsoid}
We begin by proving an intermediate result regarding the confidence ellipsoid.
\begin{theorem}\label{thm:main-confidence-interval}
Under the same assumptions as \Cref{thm:main-regret}
with probability at least $1 - \delta$ and for all $t \geq 0,$
$\bw_\star$ lies in the set
\begin{align*}
\wt{C}_t &\triangleq\left\{\bw : \normsmall{\bw - \abw_{t}}_{\abA_{t}} \leq \wt{\beta}_t\right\}
\end{align*}
with
\begin{align*}
\wt{\beta}_t & \triangleq
2\xi\sqrt{\alpha\log(\kappa^2t)\left(\sum_{s=1}^t \wt{\sigma}_{t}^2(\bx_s)\right) + \log\left(\frac{1}{\delta}\right)}
+ \left(1 + \frac{1}{\sqrt{1-\varepsilon}}\right)\sqrt{\lambda}\fnorm.
\end{align*}
\end{theorem}
\begin{proof}
For simplicity, we omit the subscript $t$. We begin by  noticing that
\begin{align*}
(\abw - \bw_\star)^\transp\abA(\abw - \bw_\star)
&=(\abw - \bw_\star)^\transp\abA(\abA^{-1}\aphimat^\transp\by - \bw_\star)\\
&=(\abw - \bw_\star)^\transp\abA(\abA^{-1}\aphimat^\transp(\phimat\bw_\star + \eta - \bw_\star)\\
&=(\abw - \bw_\star)^\transp\abA(\underbrace{\abA^{-1}\aphimat^\transp\phimat\bw_\star - \bw_\star}_{\text{bias}}) + (\abw - \bw_\star)^\transp\abA^{1/2}\underbrace{\abA^{-1/2}\aphimat^\transp\eta}_{\text{variance}}.
\end{align*}

\noindent
\textbf{Bounding the bias.}
We first focus on the first term, which is difficult to analyze due to the mismatch $\aphimat^\transp\phimat$. We have that
\begin{align*}
\abA(\abA^{-1}\aphimat^\transp\phimat\bw_\star - \bw_\star)
&=\aphimat^\transp\phimat\bw_\star - \aphimat^\transp\aphimat\bw_\star - \lambda\bw_\star\\
&=\aphimat^\transp\phimat(\bI - \bP)\bw_\star + \aphimat^\transp\phimat\bP\bw_\star - \aphimat^\transp\aphimat^\bw_\star - \lambda\bw_\star\\
&=\aphimat^\transp\phimat(\bI - \bP)\bw_\star - \lambda\bw_\star.
\end{align*}
Therefore,
\begin{align*}
(\abw - \bw_\star)^\transp
\abA(\abA^{-1}\aphimat^\transp\phimat\bw_\star - \bw_\star)
&= (\abw - \bw_\star)^\transp\aphimat^\transp\phimat(\bI - \bP)\bw_\star
- \lambda(\abw - \bw_\star)^\transp\bw_\star\\
&\leq \normsmall{\abw - \bw_\star}_{\abA}\left(\normsmall{\abA^{-1/2}\aphimat^\transp\phimat(\bI - \bP)\bw_\star} + \lambda\normsmall{\bw^{*}}_{\abA^{-1}}\right)\\
&\leq \normsmall{\abw - \bw_\star}_{\abA}\left(\normsmall{\abA^{-1/2}\aphimat^\transp\phimat(\bI - \bP)\bw_\star} + \tfrac{\lambda}{\sqrt{\lambda}}\normsmall{\bw^{*}}\right).
\end{align*}
Then, we have that
\begin{align*}
\normsmall{\abA^{-1/2}\aphimat^\transp\phimat(\bI - \bP)\bw_\star}
&\leq \normsmall{\abA^{-1/2}\aphimat^\transp}\normsmall{\phimat(\bI - \bP)}\normsmall{\bw_\star}\\
&\leq \sqrt{\lambda_{\max}(\aphimat\abA^{-1}\aphimat^\transp)}\sqrt{\lambda_{\max}(\phimat(\bI - \bP)^2\phimat^\transp)}\normsmall{\bw_\star}.
\end{align*}
It is easy to see that
\begin{align*}
\lambda_{\max}(\aphimat\abA^{-1}\aphimat^\transp)
= \lambda_{\max}(\aphimat(\aphimat^\transp\aphimat + \lambda\bI)^{-1}\aphimat^\transp)
\leq 1.
\end{align*}
To bound the other term we use the following result by \citet{calandriello2018statistical}.
\begin{proposition}\label{prop:proj-residual}
If $\coldict_t$ is $\varepsilon$-accurate w.r.t.\,$\phimat_t$, then
\begin{align*}
\bI - \bP_t \preceq \bI - \phimat_t\selmatrix_t(\selmatrix_t^\transp\phimat_t^\transp\phimat_t\selmatrix_t + \lambda\bI)^{-1}\selmatrix_t^\transp\phimat_t^\transp \preceq \frac{\lambda}{1-\varepsilon}(\phimat_t\phimat_t^\transp + \lambda\bI)^{-1}.
\end{align*}
\end{proposition}
Since from \Cref{thm:kors-complexity}, we have that
$\coldict_t$ is $\varepsilon$-accurate, by \Cref{prop:proj-residual}, we have that
\begin{align*}
\phimat(\bI - \bP)^2\phimat^\transp
=\phimat(\bI - \bP)\phimat^\transp
\preceq \frac{\lambda}{1-\varepsilon}\phimat(\phimat^\transp\phimat + \lambda\bI)^{-1}\phimat^\transp
\preceq \frac{\lambda}{1-\varepsilon}\bI.
\end{align*}
Putting it all together, we obtain
\begin{align*}
(\abw - \bw_\star)^\transp\abA(\abA^{-1}\aphimat^\transp\phimat\bw_\star - \bw_\star)
&\leq \left(1 + \frac{1}{\sqrt{1-\varepsilon}}\right)\normsmall{\abw - \bw_\star}_{\abA}\sqrt{\lambda}\normsmall{\bw_\star}.
\end{align*}

\noindent
\textbf{Bounding the variance.}
We use the the following self-normalized martingale concentration inequality
by \citet{abbasi2011improved}. It can be trivially extended to RKHSs in the case of finite sets such as our $\armset$. Note that if the reader is interested in infinite sets, \citet{chowdhury2017kernelized} provide a generalization with slightly worse constants. 
\begin{proposition}[\citealp{abbasi2011improved}]\label{prop:yassin-conc}
Let $\{\filtration_t\}^{\infty}_{t=0}$ be a filtration, let $\{\eta_t\}_{t=1}^{\infty}$ be a real-valued stochastic process such that $\eta_t$ is $\filtration_t$-measurable and zero-mean $\xi$-subgaussian; let $\{\phimat_t\}_{t=1}^{\infty}$ be an $\rkhs$-valued stochastic process such that $\phimat_t$ is $\filtration_{t-1}$-measurable, and let $\bI$ be the identity operator on $\rkhs$. For any $t \geq 1$, define
\[
\bA_t = \phimat_t^\transp\phimat_t + \lambda\bI \quad \text{and} \quad \bV_t = \phimat_t^\transp\eta_t.
\]
Then, for any $\delta > 0$, with probability at least $1-\delta$, for all $t \geq 0$,
\begin{align*}
\normsmall{\bV_t}_{\bA_t^{-1}}^2 \leq 2\xi^2\log\left(\frac{\det\left(\bA_t/\lambda\right)}{\delta}\right)\cdot
\end{align*}
\end{proposition}
\todod{Prove it forreal for RKHSs, in theory proved only for Euclidean}
Recalling the definition of $\alpha \geq 1$ from \Cref{thm:kors-complexity}, we reformulate
\begin{align*}
(\abw - \bw_\star)^\transp\abA^{1/2}\abA^{-1/2}\aphimat\eta
&\leq \normsmall{\abw - \bw_\star}_{\abA}\normsmall{\aphimat\eta}_{\abA^{-1}}\\
&= \normsmall{\abw - \bw_\star}_{\abA}\normsmall{\aphimat^\transp\eta}_{(\aphimat^\transp\aphimat + \lambda\bI)^{-1}}\\
&= \normsmall{\abw - \bw_\star}_{\abA}\normsmall{\aphimat^\transp\eta/\lambda}_{(\aphimat^\transp\aphimat/\lambda + \bI)^{-1}}.
\end{align*}
We now make a remark that requires temporal notation. Note that
we cannot directly apply \Cref{prop:yassin-conc} to $\aphimat_t\eta_t = \bP_t\phimat_t\eta_t$.
In particular, for $s<t$ we have that $\aphimat_s\eta_s = \bP_t\phimat_s\eta_s$
is not $\filtration_{s-1}$ measurable, since $\bP_t$ depends on all randomness
up to time $t$.
However, since $\bP_t$ is always a projection matrix we know that the variance
of the projected process is bounded by the variance of the original process, in particular,
\begin{align*}
\normsmall{\aphimat^\transp\eta/\lambda}_{(\aphimat^\transp\aphimat/\lambda + \bI)^{-1}}
&= \sqrt{\eta^\transp\aphimat(\aphimat^\transp\aphimat/\lambda + \bI)^{-1}\aphimat^\transp\eta/\lambda}
= \sqrt{\eta^\transp\aphimat\aphimat^\transp(\aphimat\aphimat^\transp/\lambda + \bI)^{-1}\eta/\lambda}\\
&\stackrel{(a)}{=} \sqrt{\eta^\transp(\bI - \lambda(\aphimat\aphimat^\transp/\lambda + \bI)^{-1})\eta/\lambda}
= \sqrt{\eta^\transp(\bI - \lambda(\phimat\bP\phimat^\transp/\lambda + \bI)^{-1})\eta/\lambda}\\
&\stackrel{(b)}{\leq} \sqrt{\eta^\transp(\bI - \lambda(\phimat\phimat^\transp/\lambda + \bI)^{-1})\eta/\lambda}
\stackrel{(c)}{=}\normsmall{\phimat^\transp\eta/\lambda}_{(\phimat^\transp\phimat/\lambda + \bI)^{-1}},
\end{align*}
where in $(a)$ we added and subtracted $\lambda\bI$ from $\aphimat\aphimat^\transp$,
in $(b)$ we used the fact that $\normsmall{\bP} \leq 1$ for all
projection matrices, and in $(c)$ we reversed the reformulation from $(a)$.
We can finally use \Cref{prop:yassin-conc} to obtain
\begin{align*}
\normsmall{\phimat^\transp\eta/\lambda}_{(\phimat^\transp\phimat/\lambda + \bI)^{-1}}
&\leq \sqrt{2\xi^2\log(\Det(\phimat^\transp\phimat/\lambda + \bI)/\delta)}\\
&= \sqrt{2\xi^2\log(\Det(\bA/\lambda)/\delta)}.
\end{align*}
While above is a valid bound on the radius of the confidence interval,
it is still not satisfactory. In particular, we can use Sylvester's identity to
reformulate
\begin{align*}
\logdet(\bA/\lambda)
=\logdet(\phimat^\transp\phimat/\lambda + \bI)
=\logdet(\phimat\phimat^\transp/\lambda + \bI)
=\logdet(\bK/\lambda + \bI).
\end{align*}
Computing the radius would require constructing the matrix $\bK \in \Real^{t \times t}$ and this is way
too expensive. Instead, we obtain a cheap but still a small enough upper bound
as follows,
\begin{align*}
\logdet(\bK_t/\lambda + \bI)
&\leq \Tr(\bK_t(\bK_t + \lambda\bI)^{-1})(1 + \log(\normsmall{\bK_t} + 1))\\
&\leq \Tr(\bK_t(\bK_t + \lambda\bI)^{-1})(1 + \log(\Tr{\bK_t} + 1))\\
&\leq \Tr(\bK_t(\bK_t + \lambda\bI)^{-1})(1 + \log(\kappa^2t + 1))\\
&= (1 + \log(\kappa^2t + 1))\sum_{s=1}^t \sigma_{t}^2(\bx_s)\\
&\leq \alpha(1 + \log(\kappa^2t + 1))\sum_{s=1}^t \wt{\sigma}_{t}^2(\bx_s)\\
&\leq 2\alpha\log(\kappa^2t)\sum_{s=1}^t \wt{\sigma}_{t}^2(\bx_s),
\end{align*}
where $\wt{\sigma}_{t}^2(\bx_s)$ can be computed efficiently and it is actually already done
by the algorithm at every step!
Putting it all together, we get that
\begin{align*}
\normsmall{\abw - \bw_\star}_{\abA}
&\leq 2\xi\sqrt{\alpha\log(\kappa^2t)\left(\sum_{s=1}^t \wt{\sigma}_{t}^2(\bx_s)\right) + \log(1/\delta)}
+ \left(1 + \frac{1}{\sqrt{1-\varepsilon}}\right)\sqrt{\lambda}\normsmall{\bw_\star}\\
&\leq 2\xi\sqrt{\alpha\log(\kappa^2t)\left(\sum_{s=1}^t \wt{\sigma}_{t}^2(\bx_s)\right) + \log(1/\delta)}
+ \left(1 + \frac{1}{\sqrt{1-\varepsilon}}\right)\sqrt{\lambda}\fnorm \triangleq \wt{\beta}_t.
\end{align*}\end{proof}
\subsection{Bounding the regret}
The regret analysis is straightforward. Assume that $\bw_\star \in \wt{C}_t$ is satisfied
(i.e., the event from \Cref{thm:main-confidence-interval} holds) and remember that by the definition,
$\phivec_t \triangleq \argmax_{\bx_i in \armset}\max_{\bw \in \wt{\bC}_t} \phivec_i^\transp\bw$.
We also define $\wb{\bw}_{t,i} \triangleq \argmax_{\bw \in \wt{\bC}_t} \phivec_i^\transp\bw$ as the auxiliary vector 
which encodes the optimistic behaviour of the algorithm. With a slight abuse of notation,
we also use $\star$ as a subscript to indicate the (unknown) index of the optimal arm,
so that $\wb{\bw}_{t,\star} \triangleq \argmax_{\bw \in \wt{\bC}_t} \phivec_\star^\transp\bw$. Since $\bw_\star \in \wt{C}_t,$ we have that
\begin{align*}
\phivec_t^\transp\wb{\bw}_{t,t} \geq {\phivec_\star}\wb{\bw}_{t,\star} \geq {\phivec_\star}\bw_{*}.
\end{align*}
We can now bound the instantaneous regret $r_t$ as
\begin{align*}
r_t &= \phivec_\star^\transp\bw_\star - \phivec_t^\transp\bw_\star
\leq \phivec_t^\transp\wb{\bw}_{t,t} - \phivec_t^\transp\bw_\star\\
&= \phivec_t^\transp(\wb{\bw}_{t,t} - \wh{\bw}_{t}) + \phivec_t^\transp(\wh{\bw}_{t} - \bw_\star)\\
&= \phivec_t^\transp\abA_{t}^{-1/2}\abA_{t}^{1/2}(\wb{\bw}_{t,t} - \wh{\bw}_{t}) + \phivec_t^\transp\abA_{t-1}^{-1/2}\abA_{t}^{1/2}(\wh{\bw}_{t} - \bw_\star)\\
&\leq \sqrt{\phivec_t^\transp\abA_{t}^{-1}\phivec_t}\left(\normsmall{\wb{\bw}_{t,t} - \wh{\bw}_{t}}_{\abA_{t}} + \normsmall{\wh{\bw}_{t} - \bw_\star}_{\abA_{t}}\right)\\
&\leq 2\wt{\beta}_t\sqrt{\phivec_t^\transp\abA_{t}^{-1}\phivec_t}.
\end{align*}
Summing over $t$ and taking the max over $\wt{\beta}_t,$ we get
\begin{align*}
R_t &\leq 2\wt{\beta}_T\sum_{t=1}^T\sqrt{\phivec_t^\transp\abA_{t}^{-1}\phivec_t}
\leq 2\wt{\beta}_T\sqrt{T}\sqrt{\sum_{t=1}^T\phivec_t^\transp\abA_{t}^{-1}\phivec_t}
\leq 2\wt{\beta}_T\sqrt{T}\sqrt{\alpha\sum_{t=1}^T\phivec_t^\transp\bA_{t}^{-1}\phivec_t}.
\end{align*}
We can now use once again \Cref{prop:logdet-to-deff}
to obtain
\begin{align*}
R_T &\leq 2\wt{\beta}_T \sqrt{\alpha T \sum_{t=1}^T\phivec_t^\transp\bA_{t}^{-1}\phivec_t}
= 2\wt{\beta}_T \sqrt{\alpha T \sum_{t=1}^T \sigma_t^2(\wt{\bx}_t)}
\leq 2\wt{\beta}_T \sqrt{2\alpha T \deff(\lambda,\wt{\bX}_T)\log(\kappa^2T)}.
\end{align*}
We can also further upper bound $\wt{\beta}_T$ as
\begin{align*}
\wt{\beta}_T &
= 2\xi\sqrt{\alpha\log(\kappa^2T)\left(\sum_{s=1}^T \wt{\sigma}_{t}^2(\bx_s)\right) + \log(1/\delta)}
+ \left(1 + \frac{1}{\sqrt{1-\varepsilon}}\right)\sqrt{\lambda}\fnorm\\
&\leq 
2\xi\sqrt{\alpha^2\log(\kappa^2T)\left(\sum_{s=1}^T \sigma_{t}^2(\bx_s)\right) + \log(1/\delta)}
+ \left(1 + \frac{1}{\sqrt{1-\varepsilon}}\right)\sqrt{\lambda}\fnorm\\
&\leq 2\xi\alpha\sqrt{\deff(\lambda,\wt{\bX}_T)\log(\kappa^2T) + \log(1/\delta)}
+ \left(1 + \frac{1}{\sqrt{1-\varepsilon}}\right)\sqrt{\lambda}\fnorm.
\end{align*}
Putting it together, we obtain
\begin{align*}
R_T
&\leq 2\left(2\xi\alpha\sqrt{\deff(\lambda,\wt{\bX}_T)\log(\kappa^2T) + \log(1/\delta)}\right)\sqrt{2\alpha T \deff(\lambda,\wt{\bX}_T)\log(\kappa^2T)}\\
&\quad + 2\left(\left(1 + \frac{1}{\sqrt{1-\varepsilon}}\right)\sqrt{\lambda}\fnorm\right)\sqrt{2\alpha T \deff(\lambda,\wt{\bX}_T)\log(\kappa^2T)}\\
&\leq 2\xi(2\alpha)^{3/2}\left(\deff(\lambda,\wt{\bX}_T)\log(\kappa^2T) + \log(1/\delta)\right)
 + 2\left(2\sqrt{\alpha}\sqrt{\lambda}\fnorm\right)\sqrt{2\alpha T \deff(\lambda,\wt{\bX}_T)\log(\kappa^2T)}\\
&\leq 2(2\alpha)^{3/2}\left(\sqrt{T}\xi\deff(\lambda,\wt{\bX}_T)\log(\kappa^2T) + \sqrt{T}\log(1/\delta) +\sqrt{T \lambda \fnorm^2 \deff(\lambda,\wt{\bX}_T)\log(\kappa^2T)}\right).
\end{align*}
 
\end{document}